\documentclass[11pt]{article}

\usepackage{times}
\usepackage{fullpage}

\usepackage{graphicx}
\usepackage{subfig}

\usepackage{algorithm}
\usepackage{algorithmic}

\usepackage{pdfsync}
\usepackage{verbatim}

\usepackage{amsmath}
\usepackage{amsthm}
\usepackage{microtype}

\newtheorem{lemma}{Lemma}
\newtheorem{theorem}{Theorem}

\newtheorem{definition}{Definition}

\newcommand{\littleheader}[1]{\textbf{#1:}}

\newcommand{\R}{\mathbf{R}}

\newcommand{\cost}{\mathrm{cost}}

\newcommand{\NP}{\textbf{NP}}
\newcommand{\A}{\mathcal{A}}

\newcommand{\x}{\mathbf{x}}

\newcommand{\D}{D}

\title{Distributed $k$-Means and $k$-Median Clustering on General Topologies}

\author{Maria Florina Balcan\thanks{
Georgia Institute of Technology,
\texttt{\small ninamf@cc.gatech.edu}}
\qquad
Steven Ehrlich\thanks{
Georgia Institute of Technology,
\texttt{\small sehrlich@cc.gatech.edu}}
\qquad
Yingyu Liang\thanks{
Georgia Institute of Technology,
\texttt{\small yliang39@gatech.edu}}
}

\date{}

\begin{document}

\maketitle

\begin{abstract}
This paper provides new algorithms for distributed clustering for two popular center-based objectives, $k$-median and $k$-means.
These algorithms have provable guarantees and improve communication complexity over existing approaches.
Following a classic approach in clustering by \cite{har2004coresets}, we reduce the problem of finding a clustering with low cost
to the problem of finding a coreset of small size.
We provide a distributed method for constructing a global coreset which improves over the previous methods by reducing the communication complexity,
and which works over general communication topologies.
Experimental results on large scale data sets show that this approach outperforms other coreset-based distributed clustering algorithms.
\end{abstract}

\section{Introduction}

Most classic clustering algorithms are designed for the centralized setting, but in recent years data has become distributed over different locations,
such as distributed databases~\cite{olston2003adaptive,corbett2012spanner}, images and videos over networks~\cite{mitra2011characterizing}, surveillance~\cite{greenhill2007distributed}
and sensor networks~\cite{considine2004approximate,greenwald2004power}.
In many of these applications the data is inherently distributed because, as in sensor networks, it is collected at different sites. As a consequence it has become crucial to develop clustering algorithms which are effective in the distributed setting.

Several algorithms for distributed clustering have been
proposed and empirically tested. 
Some of these algorithms~\cite{forman2000distributed,tasoulis2004unsupervised,datta2005k} are direct adaptations of centralized algorithms which rely on statistics that are easy to compute in a distributed manner.
Other algorithms~\cite{januzaj2003towards,kargupta2001distributed} generate summaries of local data and transmit
them to a central coordinator which then performs the clustering algorithm.
No theoretical guarantees are provided for the clustering quality in these algorithms,
and they do not try to minimize the communication cost.
Additionally, most of these algorithms assume that the distributed nodes can communicate with all other sites or that there is a central coordinator that communicates with all other sites.

In this paper, we study the problem of distributed clustering
where the data is distributed across nodes
whose communication is restricted to the edges of an arbitrary graph.
We provide algorithms with small communication cost and provable guarantees on the clustering quality.
Our technique for reducing communication in general graphs is based on the construction of a small set of points which act as a proxy for the entire data set.

An \emph{$\epsilon$-coreset} is a weighted set of points whose cost on any set of centers is approximately the cost of the original data on those same centers
up to accuracy $\epsilon$.
Thus an approximate solution for the coreset is also an approximate solution for the original data.
Coresets have previously been studied in the centralized setting (\cite{har2004coresets,feldman2011unified}) but have also recently been used for distributed clustering as in \cite{zhang2008approximate} and as implied by \cite{feldman2012effective}.
In this work, we propose a distributed algorithm for $k$-means and $k$-median,
by which each node constructs a local portion of a global coreset.
Communicating the approximate cost of a global solution to each node is enough for the local construction, leading to low communication cost overall.
The nodes then share the local portions of the coreset, which can be done efficiently in general graphs using a message passing approach.

More precisely, in Section \ref{sec:coreset}, we propose a distributed coreset construction algorithm based on local approximate solutions.
Each node computes an approximate solution for its local data,
and then constructs the local portion of a coreset using only its local data and the total cost of each node's approximation.
For $\epsilon$ constant, this builds a coreset of size $\tilde{O}(kd +nk)$ for $k$-median and $k$-means when the data lies in $d$ dimensions and is distributed over $n$ sites
~\footnote{For $k$-median and $k$-means in general metric spaces,
the bound on the size of the coreset can be obtained by replacing $d$ with the logarithm of the total number of points.
The analysis for general metric spaces is largely the same as that for $d$ dimensional Euclidean space,
so we will focus on Euclidean space and point out the difference when needed.}.
If there is a central coordinator among the $n$ sites, then clustering can be performed on the coordinator by collecting the local portions of the coreset
with a communication cost equal to the coreset size $\tilde{O}(kd +nk)$.
For distributed clustering over general connected topologies, we propose an algorithm based on the distributed coreset construction
and a message-passing approach, whose communication cost improves over previous coreset-based algorithms.
We provide a detailed comparison below.

Experimental results on large scale data sets show that our algorithm performs well in practice. For a fixed amount of communication, our algorithm 
outperforms other coreset construction algorithms.

\smallskip
\noindent
\littleheader{Comparison to Other Coreset Algorithms}
Since coresets summarize local information they are a natural tool to use when trying to reduce communication complexity. If each node constructs an $\epsilon$-coreset on its local data, then the union of these coresets is clearly an $\epsilon$-coreset for the entire data set.  Unfortunately the size of the coreset in this approach increases greatly with the number of nodes.

Another approach is the one presented in~\cite{zhang2008approximate}.
Its main idea is to approximate the union of local coresets with another coreset. 
They assume nodes communicate over a rooted tree, with each node passing its coreset to its parent.
Because the approximation factor of the constructed coreset depends on the quality of its component coresets, the accuracy a coreset needs (and thus the overall communication complexity) scales with the height of this tree.
Although it is possible to find a spanning tree in any communication network, when the graph has large diameter every tree has large height.
In particular many natural networks such as grid networks have a large diameter ($\Omega(\sqrt n)$ for grids) which greatly increases the size of coresets which must be communicated across the lower levels of the tree.
We show that it is possible to construct a global coreset with low communication overhead.
This is done by distributing the coreset construction procedure rather than combining local coresets.
The communication needed to construct this coreset is negligible -- just a single value from each data set representing the approximate cost of their local optimal clustering.
Since the sampled global $\epsilon$-coreset is the same size as any local $\epsilon$-coreset, this leads to an improvement of the communication cost over the other approaches.
See Figure~\ref{fig:coreset} for an illustration.
The constructed coreset is smaller by a factor of $n$ in general graphs, and is independent of the communication topology. This method excels in sparse networks with large diameters, where the previous approach in~\cite{zhang2008approximate} requires coresets that are quadratic in the size of the diameter for $k$-median and quartic for $k$-means; see Section~\ref{sec:clustering} for details.
\cite{feldman2012effective} also merge coresets using coreset construction, but they do so in a model of parallel computation and ignore communication costs.

Balcan et al.~\cite{balcan2012distributed} and Daume et al.~\cite{daume2012efficient} consider communication complexity questions arising when doing classification in distributed settings.
In concurrent and independent work, Kannan and Vempala~\cite{kannan2013nimble} study several optimization problems in distributed settings, including $k$-means clustering under an interesting separability assumption.

Section~\ref{sec:work} provides a review of additional related work.

\begin{figure}[!t]
\begin{center}
\centering
\subfloat[Zhang et al.\cite{zhang2008approximate} ]{\includegraphics[scale = 0.45]{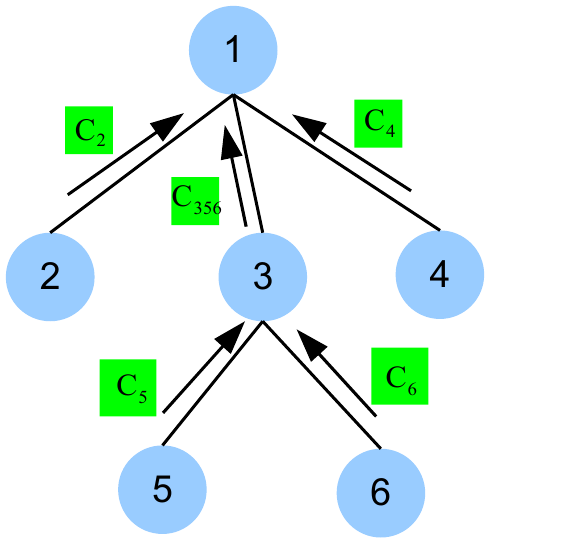}}
\hspace*{0.8in}
\subfloat[Our Construction ]{\includegraphics[scale = 0.45]{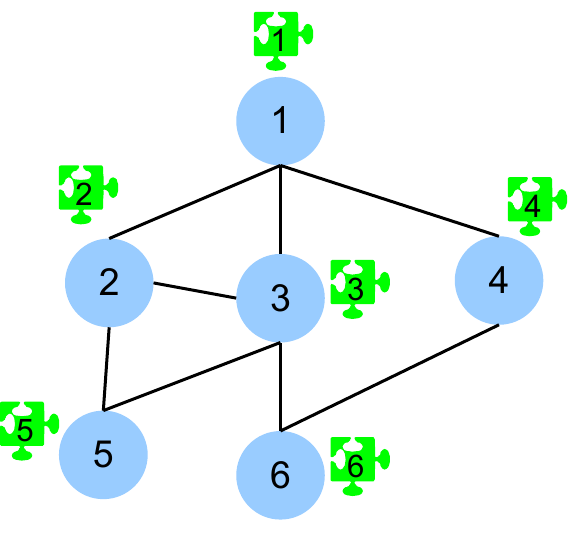}}
\end{center}
\caption{{
\textbf{(a) } Each node computes a coreset on the weighted pointset for its own data and its subtrees' coresets. 
\textbf{(b) } Local constant approximation solutions are computed, and the costs of these solutions are used to coordinate the construction of a local portion on each node. 
}}\label{fig:coreset}
\end{figure}

\section{Preliminaries}\label{sec:pre}

Let $d(p,q)$ denote the Euclidean distance between any two points $p, q \in \R^d$.
The goal of $k$-means clustering is to find a set of $k$ centers $\x=\{x_1, x_2, \dots, x_k\}$
which minimize the $k$-means cost of data set $P \subseteq \R^d$.
Here the $k$-means cost is defined as $\cost(P,\x)=\sum_{p \in P} d(p, \x)^2$ where $d(p, \x) = \min_{x \in \x} d(p, x)$.
If $P$ is a weighted data set with a weighting function $w$, then the $k$-means cost is defined as $\sum_{p \in P} w(p)d(p, \x)^2$.
Similarly, the $k$-median cost is defined as $\sum_{p \in P} d(p, \x)$.
Both $k$-means and $k$-median cost functions are known to be $\NP$-hard to minimize (see for example~\cite{ab:survey}).
For both objectives, there exist several readily available polynomial-time algorithms that achieve constant approximation solutions (see for example \cite{kanungo2002local,shi2013app}).

In the distributed clustering task, we consider a set of $n$ nodes $V=\{v_i, 1\leq i\leq n\}$ which communicate
on an undirected connected graph $G = (V,E)$ with $m=|E|$ edges.
More precisely, an edge $(v_i,v_j) \in E$ indicates that $v_i$ and $v_j$ can communicate with each other.
Here we measure the communication cost in number of points transmitted, and assume for simplicity that there is no latency in the communication.
On each node $v_i$, there is a local set of data points $P_i$, and the global data set is $P=\bigcup_{i=1}^n P_i$.
The goal is to find a set of $k$ centers $\x$ which optimize $\cost(P,\x)$
while keeping the computation efficient and the communication cost as low as possible.
Our focus is to reduce the total communication cost while preserving theoretical guarantees for approximating clustering cost.

\subsection{Coresets}

For the distributed clustering task, a natural approach to avoid broadcasting raw data is to generate a local summary of the relevant information.
If each site computes a summary for their own data set and then communicates this to a central coordinator, a solution can be computed from a much smaller amount of data,
drastically reducing the communication.

In the centralized setting, the idea of summarization with respect to the clustering task
is captured by the concept of coresets~\cite{har2004coresets,feldman2011unified}.
A coreset is a set of points, together with a weight for each point,
such that the cost of this weighted set
approximates the cost of the original data for any set of $k$ centers.
The formal definition of coresets is:

\begin{definition} [\textbf{coreset}] \label{def:coreset}
An $\epsilon$-coreset for a set of points $P$ with respect to a center-based cost function is a set of points $S$ and
a set of weights $w: S \rightarrow \R$
such that for any set of centers $\x$, 
$$ (1-\epsilon)\cost(P,\x) \leq \sum_{p\in S}w(p)\cost(p,\x) \leq (1+\epsilon)\cost(P,\x).$$
\end{definition}

In the centralized setting, many coreset construction algorithms have been proposed for $k$-median, $k$-means
and some other cost functions.
For example,  for points in $\R^d$, algorithms in~\cite{feldman2011unified} construct coresets of size $t = \tilde{O}(kd/\epsilon^4)$ for $k$-means
and coresets of size $t = \tilde{O}(kd/\epsilon^2)$ for $k$-median.
In the distributed setting, it is natural to ask whether there exists an algorithm that constructs a small coreset for the entire point set but still has low communication cost.
Note that the union of coresets for multiple data sets is a coreset for the union of the data sets.
The immediate construction of combining the local coresets from each node would produce a global coreset whose size was larger by a factor of $n$, greatly increasing the communication complexity.
We present a distributed algorithm which constructs a global coreset the same size as the centralized construction and only needs a single value\footnote{The value that is communicated is the sum of the costs of approximations to the local optimal clustering. This is guaranteed to be no more than a constant factor times larger than the optimal cost.} communicated to each node.
This serves as the basis for our distributed clustering algorithm.

\section{Distributed Coreset Construction}\label{sec:coreset}

In this section, we design a distributed coreset construction algorithm for $k$-means and $k$-median.
Note that the underlying technique can be extended to other additive clustering objectives such as $k$-line median.

To gain some intuition on the distributed coreset construction algorithm,
we briefly review the coreset construction algorithm in~\cite{feldman2011unified}
in the centralized setting.
The coreset is constructed by computing a constant approximation solution for the entire data set, and then sampling points proportional to their contributions to the cost of this solution.
Intuitively, the points close to the nearest centers can be approximately represented by the nearest centers
while points far away cannot be well represented.
Thus, points should be sampled with probability proportional to their contributions to the cost.

Directly adapting the algorithm to the distributed setting would require computing a constant approximation solution for the entire data set. 
We show that a global coreset can be constructed in a distributed fashion by estimating the weight of the entire data set with the sum of local approximations.
We first compute a local approximation solution for each local data set, and communicate the total costs of these local solutions.
Then we sample points proportional to their contributions to the cost of their local solutions.
At the end of the algorithm, the coreset consists of the sampled points and the centers in the local solutions.
The coreset points are distributed over the nodes, so we call it distributed coreset.
See Algorithm~\ref{alg:con_kmedian} for details.

\newcommand{\rw}{\qquad\ \ \ \ \ \quad}
\newcommand{\bw}{\quad\qquad}
\newcommand{\sbw}{\qquad\ \ \ \ \quad\qquad}
\newcommand{\tw}{\qquad}
\newcommand{\bulletspace}{\ \ \ }
\newcommand{\mybullet}{\bullet}
\begin{algorithm}[thbp]
\caption{Communication aware distributed coreset construction}
\label{alg:con_kmedian}
\begin{algorithmic}
\STATE{\textbf{Input:} Local datasets $\{P_i, 1\leq i \leq n\}$, parameter $t$ (number of points to be sampled).}
\STATE
\STATE{\tw \textbf{Round 1:} on each node $v_i \in V$ }
\STATE{\bw $\mybullet$ Compute a constant approximation $B_i$ for $P_i$. \\
\bw \bulletspace Communicate $\cost(P_i,B_i)$ to all other nodes.}
\STATE{\tw \textbf{Round 2:} on each node $v_i \in V$}
\STATE{ \bw $\mybullet$ Set $t_i = \frac{t \ \cost(P_i, B_i) }{\sum_{j=1}^n\cost(P_j,B_j)}$ and
$m_p = \cost(p,B_i), \forall p\in P_i$.}
\STATE{ \bw $\mybullet$ Pick a non-uniform random sample $S_i$ of $t_i$ points from $P_i$,\\
\bw \bulletspace where for every $q \in S_i$ and $p \in P_i$, we have $q = p$ with
probability $ m_p /\sum_{z\in P_i} m_z$.\\
\bw \bulletspace Let $w_q =\frac{\sum_i\sum_{z\in P_i} m_z}{t m_q}$ for each $q \in S_i$.}
\STATE{\bw $\mybullet$ For $\forall b \in B_i$, let $P_{b}= \{p \in P_i: d(p,b) = d(p,B_i)\}$,
$w_b=|P_{b}| - \sum_{q\in P_{b} \cap S}w_q.$}
\STATE
\STATE{\textbf{Output:} Distributed coreset: points $S_i \cup B_i$ with weights $\{w_q: q \in S_i\cup B_i\}$, $1\leq i\leq n$.}
\end{algorithmic}
\end{algorithm}

\begin{theorem}\label{thm:coreset}
For distributed $k$-means and $k$-median clustering on a graph, there exists an algorithm such that with probability at least $1-\delta$,
the union of its output on all nodes is an $\epsilon$-coreset for $P=\bigcup_{i=1}^n P_i$. The size of the coreset is
$O(\frac{1}{\epsilon^4}(kd \log(kd) + \log\frac{1}{\delta}) + nk\log\frac{nk}{\delta})$ for $k$-means, and
$O(\frac{1}{\epsilon^2}(kd\log(kd) + \log\frac{1}{\delta}) + nk)$ for $k$-median.
The total communication cost is $O(mn)$.
\end{theorem}

As described below, the distributed coreset construction can be achieved by using Algorithm~\ref{alg:con_kmedian} with appropriate $t$, namely $O(\frac{1}{\epsilon^4}(kd\log(kd) + \log\frac{1}{\delta}) + nk\log\frac{nk}{\delta})$ for $k$-means and $O(\frac{1}{\epsilon^2}(kd\log(kd) + \log\frac{1}{\delta}))$ for $k$-median.
The formal proofs are described in the following subsections.

\subsection{Proof of Theorem~\ref{thm:coreset}: $k$-median}
\label{sec:analysis_coreset}

The analysis relies on the definition of the dimension of a function space and a sampling lemma.

\begin{definition}[\cite{feldman2011unified}]
Let $F$ be a finite set of functions from a set $P$ to $\R_{\geq 0}$.
For $f\in F$, let $B(f,r) = \{p: f(p) \leq r\}$.
The dimension of the function space $\dim(F, P)$ is the smallest integer $d$ such that for any $G \subseteq P$,
$\bigl|\{G \cap B(f,r): f \in F, r\geq 0\}\bigr|\leq |G|^d.$
\end{definition}

Suppose we draw a sample $S$ according to $\{m_p: p\in P\}$,
namely for every $q \in S$ and every $p \in P$, we have $q=p$ with probability $\frac{m_p}{\sum_{z \in P}m_z}$.
Set the weights of the points as $w_p=\frac {\sum_{z\in P}m_z}{m_p|S|}$ for $p \in P$.
Then for any $f\in F$, the expectation of the weighted cost of $S$ equals the cost of the original data $P$:
\begin{eqnarray*}
\mathbf{E}\left[\sum_{q\in S} w_qf(q) \right] & = & \sum_{q\in S} \mathbf{E}[w_q f(q)] = \sum_{q\in S} \sum_{p\in P}\Pr[q=p] w_p f(p) \\
& = &\sum_{q\in S}\sum_{p \in P}\frac{m_p}{\sum_{z \in P}m_z}\frac {\sum_{z\in P}m_z}{m_p|S|}f(p)
= \sum_{q\in S}\sum_{p\in P}\frac{1}{|S|}f(p) = \sum_{p\in P}f(p).
\end{eqnarray*}

The following lemma shows that if the sample size is large enough, then we also have concentration for any $f \in F$.
The lemma is implicit in~\cite{feldman2011unified} and we include the proof in the appendix for completeness.

\begin{lemma}\label{lem:sampling}
Fix a set $F$ of functions $f:P \to \R_{\geq 0}$.
Let $S$ be a sample drawn i.i.d. from $P$ according to $\{m_p: p\in P\}$, namely, for every $q \in S$ and every $p \in P$, we have $q=p$ with probability $\frac{m_p}{\sum_{z \in P}m_z}$.
Let $w_p=\frac {\sum_{z\in P}m_z}{m_p|S|}$ for $p \in P$.
For a sufficiently large $c$,
if $|S| \geq \frac c {\epsilon^2} \left(\dim(F,P) \log \dim(F,P) +\log \frac 1\delta\right)$
then with probability at least $1-\delta, \forall f \in F:$
$\left| \sum_{p\in P} f(p) - \sum_{q\in S} w_q f(q)\right| \leq \epsilon \left(\sum_{p\in P}m_p \right)\left(\max_{p\in P} \frac{f(p)}{m_p}\right).$
\end{lemma}
To get a small bound on the difference between $\sum_{p\in P} f(p)$ and $\sum_{q\in S} w_q f(q)$, we need to choose $m_p$ such that $\max_{p\in P} \frac{f(p)}{m_p}$ is bounded. More precisely, if we choose $m_p=\max_{f\in F} f(p)$, then the difference is bounded by $\epsilon \sum_{p\in P}m_p $.


We first consider the centralized setting and
review how \cite{feldman2011unified} applied the lemma to construct a coreset for $k$-median as in Definition~\ref{def:coreset}.
A natural approach is to apply this lemma directly to the cost, namely, to choose $f_\x(p):=\cost(p,\x)$.
The problem is that a suitable upper bound $m_p$ is not available for $\cost(p,\x)$.
However, we can still apply the lemma to a different set of functions defined as follows.
Let $b_p$ denote the closest center to $p$ in the approximation solution. 
Aiming to approximate the error $\sum_p[\cost(p,\x) - \cost(b_p,\x)]$ rather than to approximate $\sum_p\cost(p,\x)$ directly, we define $f_\x(p):= \cost(p,\x)-\cost(b_p,\x)+\cost(p,b_p)$, where $\cost(p,b_p)$ is added so that $f_\x(p) \geq 0$.
Since $0\leq f_\x(p) \leq 2 \cost(p,b_p)$, we can apply the lemma to $f_\x(p)$ and $m_p = 2\cost(p,b_p)$.
The lemma then bounds the difference $|\sum_{p\in P} f_\x(p) - \sum_{q \in S} w_q f_\x(q)|$ by $2\epsilon \sum_{p \in P}\cost(p,b_p)$,
so we have an $O(\epsilon)$-approximation.

Note that $\sum_{p\in P} f_\x(p) - \sum_{q \in S} w_q f_\x(q)$ does not equal $\sum_{p\in P} \cost(p,\x) - \sum_{q \in S} w_q \cost(q,\x)$.
However, it equals the difference between $\sum_{p\in P} \cost(p,\x) $ and a weighted cost of the sampled points and the centers in the approximation solution.
To get a coreset as in Definition~\ref{def:coreset},
we need to add the centers of the approximation solution with specific weights to the coreset.
Then when the sample is sufficiently large, the union of the sampled points and the centers is an $\epsilon$-coreset.


Our key contribution in this paper is to show that in the distributed setting, it suffices to choose $b_p$ from the local approximation solution for the local dataset containing $p$, rather than from an approximation solution for the global dataset.
Furthermore, the sampling and the weighting of the coreset points can be done in a local manner.
In the following, we provide a formal verification of our discussion above.
We have the following lemma for $k$-median with
$F=\{f_\x: f_\x(p)= d(p,\x)-d(b_p,\x)+d(p,b_p), \x \in (\R^d)^k\}.$

\begin{lemma}\label{lem:kmedian}
For $k$-median, the output of Algorithm~\ref{alg:con_kmedian} is an $\epsilon$-coreset with probability at least $1-\delta$,
if $t \geq \frac{c}{\epsilon^2}\left(\mathrm{dim}(F, P) \log \mathrm{dim}(F, P) + \log \frac{1}{\delta}\right)$
for a sufficiently large constant $c$.
\end{lemma}
\renewcommand{\cost}{d}

\begin{proof}
We want to show that for any set of centers $\x$ the true cost for using these centers is well approximated by the cost on the weighted coreset.
Note that our coreset has two types of points: sampled points $p\in S=\cup_{i=1}^n S_i$ with weight $w_p:=\frac{\sum_{z\in P}m_z}{m_p|S|}$ and
local solution centers $b\in B=\cup_{i=1}^n B_i$ with weight $w_b:= |P_b| - \sum_{p\in S\cap P_b} w_p$.
We use $b_p$ to represent the nearest center to $p$ in the local approximation solution.
We use $P_b$ to represent the set of points having $b$ as their closest center in the local approximation solution.

As mentioned above, we construct $f_\x$ to be the difference between the cost of $p$ and the cost of $b_p$ on $\x$ so that Lemma~\ref{lem:sampling} can be applied to $f_\x$.
Note that $0 \leq f_\x(p) \leq 2 \cost(p,b_p)$ by triangle inequality,
and $S$ is sufficiently large and chosen according to weights $m_p= \cost(p,b_p)$,
so the conditions of Lemma~\ref{lem:sampling} are met.
Then we have
\begin{align*}
\D=&\Biggl| \sum_{p\in P}f_\x(p)  -\sum_{q\in S}w_q f_\x(q)\Biggr|
\leq 2 \epsilon \sum_{p\in P}m_p = 2\epsilon \sum_{p\in P}\cost(p, b_p) = 2\epsilon \sum_{i=1}^n \cost(P_i, B_i) \leq O(\epsilon)\sum_{p\in P}\cost(p,\x)
\end{align*}
where the last inequality follows from the fact that $B_i$ is a constant approximation solution for $P_i$.

Next, we show that the coreset is constructed such that $\D$ is exactly the difference between the true cost and the weighted cost of the coreset,
which then leads to the lemma.

Note that the centers are weighted such that
\begin{align}
\sum_{b\in B} w_b \cost(b,\x)&= \sum_{b\in B} |P_b|\cost(b,\x) - \sum_{b\in B}\sum_{q\in S\cap P_b} w_q \cost(b,\x) = \sum_{p\in P} \cost(b_p,\x)- \sum_{q\in S}w_q \cost(b_q,\x).\label{eqn:weight}
\end{align}

Also note that $\sum_{p\in P}m_p = \sum_{q\in S}w_q m_q$, so
\begin{align}
\D=&\Biggl| \sum_{p\in P}\left[\cost(p,\x) -\cost(b_p,\x)+ m_p\right] -\sum_{q\in S}w_q\left[ \cost(q,\x)-\cost(b_q,\x)+  m_q\right] \Biggr| \nonumber\\
=& \Biggl| \sum_{p\in P}\cost(p,\x)-\sum_{q\in S}w_q \cost(q,\x) -\biggl[\sum_{p\in P}\cost(b_p,\x)-\sum_{q\in S}w_q \cost(b_q,\x)\biggr]\Biggr|.\label{eqn:Dfull}
\intertext{By plugging (\ref{eqn:weight}) into (\ref{eqn:Dfull}), we have}
\D=&\Biggl|\sum_{p\in P} \cost(p,\x) - \sum_{q\in S}w_q \cost(q,\x) - \sum_{b\in B}w_b \cost(b,\x)\Biggr| = \Biggl|\sum_{p\in P}\cost(p,\x)-\sum_{q\in S\cup B}w_q\cost(q,\x)\Biggr|\nonumber
\end{align}
which implies the lemma.
\end{proof}

In~\cite{feldman2011unified} it is shown that~\footnote{For both $k$-median and $k$-means in general metric spaces, $\mathrm{dim}(F,P) = O(k \log |P|)$,
so the bound for general metric spaces (including Euclidean space we focus on) can be obtained by replacing $d$ with $\log |P|$.}
$\mathrm{dim}(F,P) = O(kd)$.
So by Lemma~\ref{lem:kmedian}, when $|S| \geq O\left(\frac{1}{\epsilon^2}(kd\log(kd) + \log \frac{1}{\delta})\right)$,
the weighted cost of $S\cup B$ approximates the $k$-median cost of $P$ for any set of centers,
then $(S\cup B,w)$ is an $\epsilon$-coreset for $P$.
The total communication cost is bounded by $O(mn)$, since
even in the most general case when every node only knows its neighbors,
we can broadcast the local costs with $O(mn)$ communication (see Algorithm~\ref{alg:message}).

\renewcommand{\cost}{\mathrm{cost}}
\subsection{Proof of Theorem~\ref{thm:coreset}: $k$-means}

We have for $k$-means a similar lemma that when $t=O(\frac{1}{\epsilon^4}(kd\log(kd) +\log\frac{1}{\delta})+nk \log\frac{nk}{\delta}))$,
the algorithm constructs an $\epsilon$-coreset with probability at least $1-\delta$.
The key idea is the same as that for $k$-median: we use centers $b_p$ from the local approximation solutions
as an approximation to the original data points $p$, and show that the error between the total cost
and the weighted sample cost is approximately the error between the cost of $p$ and its sampled cost
(compensated by the weighted centers), which is shown to be small by Lemma~\ref{lem:sampling}.

The key difference between $k$-means and $k$-median is that triangle inequality applies directly to the $k$-median cost.
In particular, for the $k$-median problem note that $\cost(b_p,p) = d(b_p,p)$ is an upper bound for the error of $b_p$ on any set of centers,
i.e.\ $\forall \x \in (\R^d)^k$,
$d(b_p, p) \geq |d(p,\x) - d(b_p,\x)| = |\cost(p,\x)-\cost(b_p,\x)|$
by triangle inequality.
Then we can construct $f_\x(p):= \cost(p,\x)-\cost(b_p,\x) + d(b_p, p)$ such that $h_p(\x)$ is bounded.
In contrast, for $k$-means, the error $|\cost(p,\x)-\cost(b_p,\x)| = |d(p,\x)^2 - d(b_p,\x)^2|$ does not have such an upper bound.
The main change to the analysis is that we divide the points into two categories:
good points whose costs approximately satisfy the triangle inequality (up to a factor of $1/\epsilon$) and bad points.
The good points for a fixed set of centers $\x$ are defined as
$$G(\x) = \{p \in P: |\cost(p,\x)-\cost(b_p,\x)| \leq \Delta_p\}$$
where the upper bound is $\Delta_p =\frac{\cost(p,b_p)}{\epsilon}$.
Good points we can bound as before. For bad points we can show that while the difference in cost may be larger than $\cost(p,b_p)/\epsilon$, it must still be small, namely $O(\epsilon \min\{\cost(p,\x),\cost(b_p,\x)\})$.

Formally, the functions $f_\x(p)$ are restricted to be defined only over good points:
\begin{eqnarray*}
f_\x(p) =
\begin{cases}
\cost(p,\x) - \cost(b_p,\x) + \Delta_p & \textrm{ if } p \in G(\x),\\
0 & \textrm{ otherwise.}
\end{cases}
\end{eqnarray*}
Then $\sum_{p\in P} \cost(p,\x) - \sum_{q\in S \cup B} w_q \cost(q,\x) $ is decomposed into three terms:
\begin{eqnarray}
&& \sum_{p \in P} f_\x(p) -\sum_{q \in S} w_q f_\x(q) \label{eqn:means:term1}\\
&& + \sum_{p \in P\setminus G(\x)} [\cost(p,\x) - \cost(b_p,\x) + \Delta_p]\label{eqn:means:term2}\\
&& - \sum_{q \in S\setminus G(\x)} w_q [\cost(q,\x) - \cost(b_q,\x)  + \Delta_q]\label{eqn:means:term3}
\end{eqnarray}

Lemma~\ref{lem:sampling} bounds (\ref{eqn:means:term1}) by $O(\epsilon) \cost(P,\x)$, but we need an accuracy of
$\epsilon^2$ to compensate for the $1/\epsilon$ factor in the upper bound,
resulting in a $O(1/\epsilon^4)$ factor in the sample complexity.

We begin by bounding (\ref{eqn:means:term2}).
Note that for each term in (\ref{eqn:means:term2}), $ |\cost(p,\x) - \cost(b_p,\x)| >\Delta_p$ since $p \not\in G(\x)$.
Furthermore, $p \not\in G(\x)$ only when $p$ and $b_p$ are close to each other and far away from $\x$.
In Lemma~\ref{lem:techBound} we use this to show that
$|\cost(p,\x) - \cost(b_p,\x)| \leq O(\epsilon)\min\{\cost(p,\x), \cost(b_p,\x)\}.$
The details are presented in the appendix.

Using Lemma~\ref{lem:techBound}, (\ref{eqn:means:term2}) can be bounded by
$O(\epsilon)\sum_{p \in P\setminus G(\x)} \cost(p,\x) \leq O(\epsilon) \cost(P, \x).$

Similarly, by the definition of $\Delta_q$ and Lemma~\ref{lem:techBound}, (\ref{eqn:means:term3}) is bounded by
\begin{eqnarray*}
(\ref{eqn:means:term3})&\leq&\sum_{q\in S \setminus G(\x)} 2 w_q |\cost(q,\x) - \cost(b_q,\x)| \leq O(\epsilon ) \sum_{q \in S \setminus G(\x)} w_q\ \cost(b_q,\x) \\
& \leq & O(\epsilon )  \sum_{b\in B}\left(\sum_{q \in P_b \cap S} w_q\right) \cost(b,\x).
\end{eqnarray*}
Note that the expectation of $\sum_{q\in P_{b}\cap S} w_q$ is $|P_b|$.
By a sampling argument (Lemma~\ref{lem:weight}),
if $t\geq O(nk\log\frac{nk}{\delta})$, then $\sum_{q\in P_{b}\cap S} w_q \leq 2|P_{b}|$.
Then (\ref{eqn:means:term3}) is bounded by
$O(\epsilon) \sum_{b \in B} \cost(b,\x) |P_{b}|= O(\epsilon) \sum_{p \in P} \cost(b_p,\x)$
where $\sum_{p \in P} \cost(b_p,\x)$ is at most a constant factor more than the optimum cost.

Since each of (\ref{eqn:means:term1}),(\ref{eqn:means:term2}), and (\ref{eqn:means:term3}) is $O(\epsilon)\cost(P,\x)$, we know that their sum is the same magnitude.
Combining the above bounds, we have
$\left|\cost(P,\x) -  \sum_{q\in S\cup B} w_q \cost(q,\x) \right|
 \leq  O(\epsilon)\cost(P,\x).$
The proof is then completed by choosing a suitable $\epsilon$,
and bounding $\mathrm{dim}(F, P) = O(kd)$ as in~\cite{feldman2011unified}.

\section{Effect of Network Topology on Communication Cost}\label{sec:clustering}

In the previous section, we presented a distributed coreset construction algorithm.
The coreset constructed can then be used as a proxy for the original data, and we can run any distributed clustering
algorithm on it. In this paper, we discuss the approach of simply collecting all local portions of the distributed coreset
and run non-distributed clustering algorithm on it.
If there is a central coordinator in the communication graph, then we can simply send the local portions of the coreset
to the coordinator which can perform the clustering task. The total communication cost is just the size of the coreset.

In this section, we consider the distributed clustering tasks where the nodes are arranged in some arbitrary connected topology, and can only communicate with their neighbors.
We propose a message passing approach for globally sharing information, 
and use it for collecting information for coreset construction and sharing the local portions of the coreset. 
We also consider the special case when the graph is a rooted tree.

\begin{algorithm}[!tbhp]
\caption{Distributed clustering on a graph}
\label{alg:kmedian_graph}
\begin{algorithmic}
\STATE{\textbf{Input:} $\{P_i, 1 \leq i \leq n\}$: local datasets; $\{N_i,1 \leq i \leq n\}$: the neighbors of $v_i$; $\A_\alpha$: an $\alpha$-approximation algorithm for weighted clustering instances.}
\STATE
\STATE{\tw \textbf{Round 1:} on each node $v_i$ }
\STATE{\bw $\mybullet$ Construct its local portion $D_i$ of an $\epsilon/2$-coreset by Algorithm~\ref{alg:con_kmedian}, \\
\bw \bulletspace using Message-Passing for communicating the local costs.}
\STATE{\tw \textbf{Round 2:} on each node $v_i$ }
\STATE{\bw $\mybullet$ Call Message-Passing($D_i, N_i$).}
\STATE{\bw $\mybullet$ $\x=\A_\alpha(\bigcup_{j} D_j)$.}
\STATE
\STATE{\textbf{Output:} $\x$}
\end{algorithmic}
\end{algorithm}

\begin{algorithm}[!tbhp]
\caption{Message-Passing($I_i$, $N_i$)}
\label{alg:message}
\begin{algorithmic}
\STATE{\textbf{Input:} $I_i$ is the message, $N_i$ are the neighbors.}
\STATE
\STATE{\bw $\mybullet$ Let $R_i$ denote the information received.}
\STATE{\bw \bulletspace Initialize $R_i = \{I_i\}$, and send $I_i$ to all the neighbors.}
\STATE{\bw $\mybullet$ While $R_i \neq \{I_j, 1\leq j\leq n\}$:\\
\bw \bulletspace \tw If receive message $I_j \not\in R_i$, \\
\bw \bulletspace \tw \tw $R_i = R_i \cup \{I_j\}$ and send $I_j$ to all the neighbors.
}
\end{algorithmic}
\end{algorithm}

\subsection{General Graphs}
%
%
%
%

We now present the main result for distributed clustering on graphs. 
\begin{theorem}\label{thm:clustering}
Given an $\alpha$-approximation algorithm for weighted $k$-means ($k$-median respectively) as a subroutine, there exists an algorithm
that  with probability at least $1-\delta$ outputs a $(1+\epsilon)\alpha$-approximation solution for distributed $k$-means ($k$-median respectively) clustering.
The total communication cost is $O(m (\frac{1}{\epsilon^4}(kd\log(kd) + \log\frac{1}{\delta}) + nk\log\frac{nk}{\delta}))$ for $k$-means,
and $O(m (\frac{1}{\epsilon^2}(kd \log(kd) + \log\frac{1}{\delta}) + nk))$ for $k$-median.
\end{theorem}

\begin{proof}
The details are presented in Algorithm~\ref{alg:kmedian_graph}.
By Theorem~\ref{thm:coreset}, the output of Algorithm~\ref{alg:con_kmedian} is a coreset.
Observe that in Algorithm~\ref{alg:message}, for any $j$, $I_j$ propagates on the graph in a breadth-first-search style,
so at the end every node receives $I_j$.
This holds for all $1\leq j\leq n$, so all nodes has a copy of the coreset at the end,
and thus the output is a $(1+\epsilon)\alpha$-approximation solution.

Also observe that in Algorithm~\ref{alg:message}, for any node $v_i$ and $j \in [n]$, $v_i$ sends out $I_j$ once,
so the communication of $v_i$ is $|N_i| \times \sum_{j=1}^n |I_j|$.
The communication cost of Algorithm~\ref{alg:message} is $O(m\sum_{j=1}^n |I_j|)$.
Then the total communication cost of Algorithm~\ref{alg:kmedian_graph} follows from the size of the coreset constructed.
\end{proof}

In contrast, an approach where each node constructs an $\epsilon$-coreset for $k$-means and sends it to the other nodes incurs communication cost of $\tilde{O}(\frac{mnkd}{\epsilon^4})$.
Our algorithm significantly reduces this.


\subsection{Rooted Trees}\label{sec:clustering_tree}
Our algorithm can also be applied on a rooted tree, and compares favorably to other approaches involving coresets~\cite{zhang2008approximate}.
We can restrict message passing to operating along this tree, leading to the following theorem.

\begin{theorem}\label{thm:clustering_node}
Given an $\alpha$-approximation algorithm for weighted $k$-means ($k$-median respectively) as a subroutine, there exists an algorithm
that with probability at least $1-\delta$ outputs a $(1+\epsilon)\alpha$-approximation solution for
distributed $k$-means ($k$-median respectively) clustering on a rooted tree of height $h$.
The total communication cost is $O(h (\frac{1}{\epsilon^4}(kd \log(kd) + \log\frac{1}{\delta}) + nk\log\frac{nk}{\delta}))$ for $k$-means,
and $O(h (\frac{1}{\epsilon^2}(kd \log(kd) + \log\frac{1}{\delta}) + nk))$ for $k$-median.
\end{theorem}
\begin{proof}
We can construct the distributed coreset using Algorithm~\ref{alg:con_kmedian}.
In the construction, the costs of the local approximation solutions are sent from every node to the root, and the sum is sent to every node by the root.
After the construction, the local portions of the coreset are sent from every node to the root.
A local portion $D_i$ leads to a communication cost of $O(|D_i| h)$, so the total communication cost is $O(h\sum_{i=1}^n |D_i|)$.
Once the coreset is constructed at the root, the $\alpha$-approximation algorithm can be applied centrally, and the results can be sent back to all nodes.
\end{proof}

Our approach improves the cost of $\tilde{O}(\frac{nh^4kd}{\epsilon^4})$ for $k$-means and the cost of $\tilde{O}(\frac{nh^2kd}{\epsilon^2})$ for $k$-median in~\cite{zhang2008approximate}~\footnote{
 Their algorithm used
 coreset construction as a subroutine.
 The construction algorithm they used builds coreset of size $\tilde{O}(\frac{nkh}{\epsilon^d}\log |P|)$.
 Throughout this paper, when we compare to \cite{zhang2008approximate} we assume they use the coreset construction technique of \cite{feldman2011unified} to reduce their coreset size and communication cost.
 }.
The algorithm in~\cite{zhang2008approximate} builds on each node a coreset for the union of coresets from its children,
and thus needs $O(\epsilon/h)$ accuracy to prevent the accumulation of errors.
Since the coreset construction subroutine has quadratic dependence on $1/\epsilon$ for $k$-median (quartic for $k$-means),
the algorithm then has quadratic dependence on $h$ (quartic for $k$-means).
Our algorithm does not build coreset on top of coresets,
resulting in a better dependence on the height of the tree $h$.

In a general graph, any rooted tree will have its height $h$ at least as large as half the diameter.
For sensors in a grid network, this implies $h = \Omega(\sqrt n)$.
In this case, our algorithm gains a significant improvement over existing algorithms.

\section{Experiments}\label{sec:exp}
In our experiments we seek to determine whether our algorithm is effective for the clustering tasks and how it compares to the other distributed coreset algorithms~\footnote{Our theoretical analysis shows that our algorithm has better bounds on the communication cost.
Since the bounds are from worst-case analysis, it is meaningful to verify that our algorithm also empirically
outperforms other distributed coreset algorithms.}.
We present the $k$-means cost of the solution produced by our algorithm with varying communication cost,
and compare to those of other algorithms when they use the same amount of communication.

\smallskip
\noindent
\littleheader{Data sets}
Following the setup of~\cite{zhang2008approximate,bahmani2012scalable}, for the synthetic data
we randomly choose $k=5$ centers from the standard Gaussian distribution in $\R^{10}$,
and sample equal number of $20,000$ points from the Gaussian distribution
around each center.
Note that, as in~\cite{zhang2008approximate,bahmani2012scalable}, we use the cost of the centers as a baseline for comparing the clustering quality.
We choose the following real world data sets from~\cite{Bache+Lichman:2013}:
Spam (4601 points in $\R^{58}$), Pendigits (10992 points in $\R^{16}$), Letter (20000 points in $\R^{16}$), and ColorHistogram of the Corel Image data set (68040 points in $\R^{32}$).
We use $k=10$ for these data sets.
We further choose YearPredictionMSD (515345 points in $\R^{90}$) for larger scale experiments,
and use $k=50$ for this data set.

\smallskip
\noindent
\littleheader{Experimental Methodology}
To transform the centralized clustering data sets into distributed data sets we first generate a communication graph
connecting local sites, and then partition the data into local data sets.
To evaluate our algorithm, we consider several network topologies and partition methods.

The algorithms are evaluated on three types of communication graphs: random, grid, and preferential.
The random graphs are Erd{\"o}s-Renyi graphs $G(n,p)$ with $p=0.3$, i.e.\ they are generated by including each potential edge independently with probability $0.3$.
The preferential graphs are generated according to the preferential attachment mechanism in the Barab\'{a}si-Albert model~\cite{Reka:RevModPhys2002}.
For data sets Spam, Pendigits, and Letter, we use random/preferential graphs with $10$ sites and $3 \times 3$ grid graphs.
For synthetic data set and ColorHistogram, we use random/preferential graphs with $25$ sites and $5 \times 5$ grid graphs.
For large data set YearPredictionMSD, we use random/preferential graphs with $100$ sites and $10 \times 10$ grid graphs.

The data is then distributed over the local sites.
When the communication network is a random graph, we consider three partition methods:
uniform, similarity-based, and weighted.
In the uniform partition, each data point in the global data set is assigned to the local sites with equal probability.
In the similarity-based partition, each site has an associated data point randomly selected from the global data.
Each data point in the global data is then assigned to the site with probability proportional to its similarity to the associated point of the site,
where the similarities are computed by Gaussian kernel function.
In the weighted partition, each local site is assigned a weight chosen by $|N(0,1)|$
and then each data point is distributed to the local sites with probability proportional to the site's weight.
When the network is a grid graph, we consider the similarity-based and weighted partitions.
When the network is a preferential graph, we consider the degree-based partition, where each point is assigned with probability proportional to the site's degree.

To measure the quality of the coreset generated, we run Lloyd's algorithm on the coreset and the global data respectively to get two solutions,
and compute the ratio between the costs of the two solutions over the global data.
The average ratio over 30 runs is then reported.
We compare our algorithm with COMBINE, the method of combining a coreset from each local data set, and with the algorithm of~\cite{zhang2008approximate} (Zhang et al.).
When running the algorithm of Zhang et al., we restrict the general communication network to a spanning tree by picking a root uniformly at random and performing a breadth first search.

\newcommand{\figWidth}{0.33\textwidth}
\newcommand{\betweenWidth}{.0in}

\begin{figure}[t]
\begin{center}
\centering
    \subfloat[random graph, uniform]{\includegraphics[width=\figWidth]{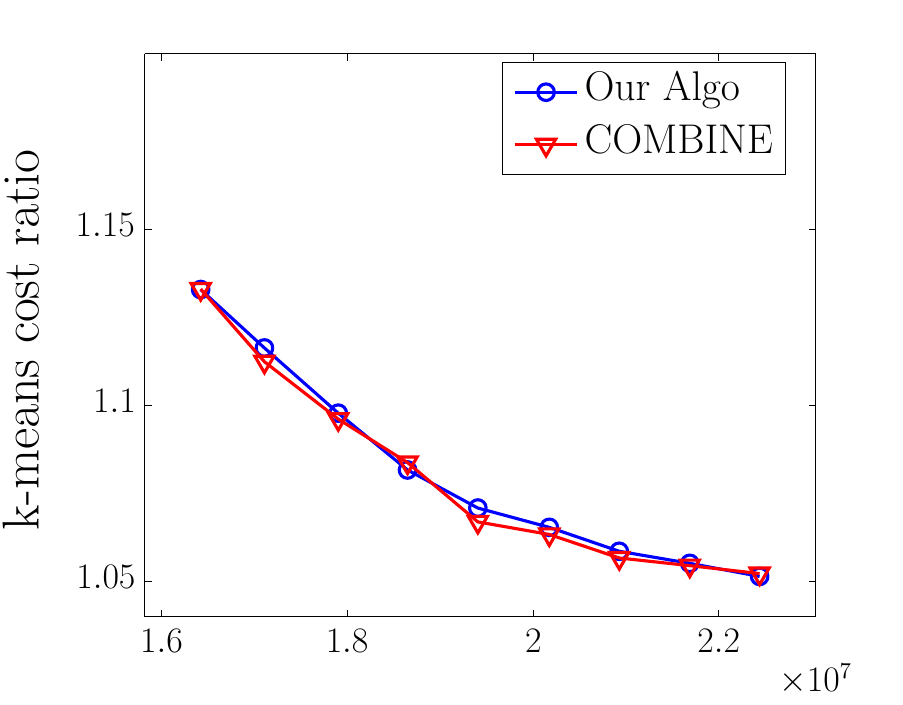}}
    \hspace*{\betweenWidth}
    \subfloat[random graph, similarity-based]{\includegraphics[width=\figWidth]{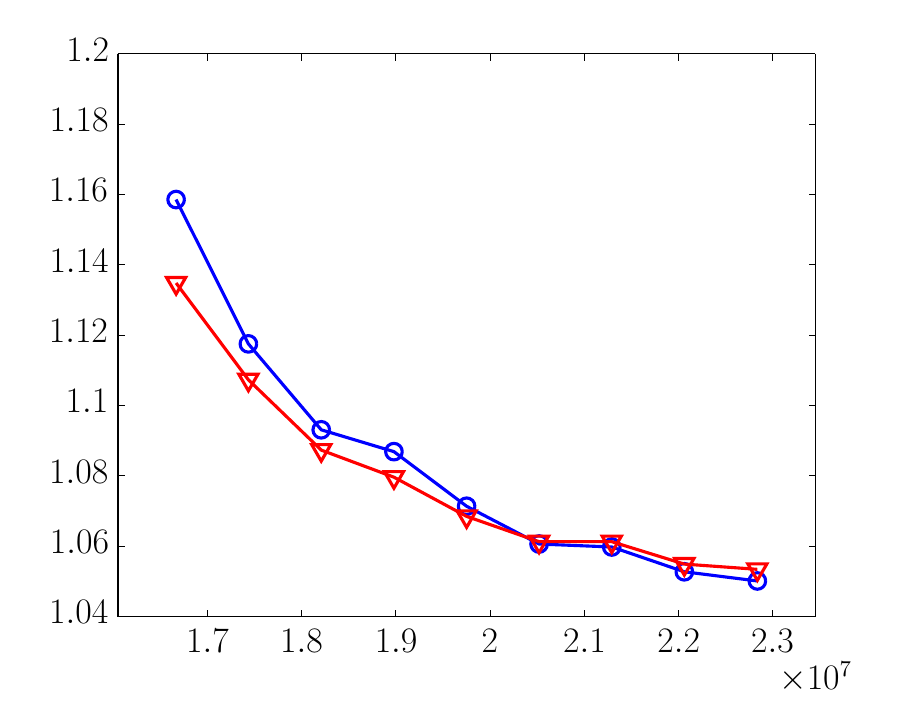}}
    \hspace*{\betweenWidth}
    \subfloat[random graph, weighted]{\includegraphics[width=\figWidth]{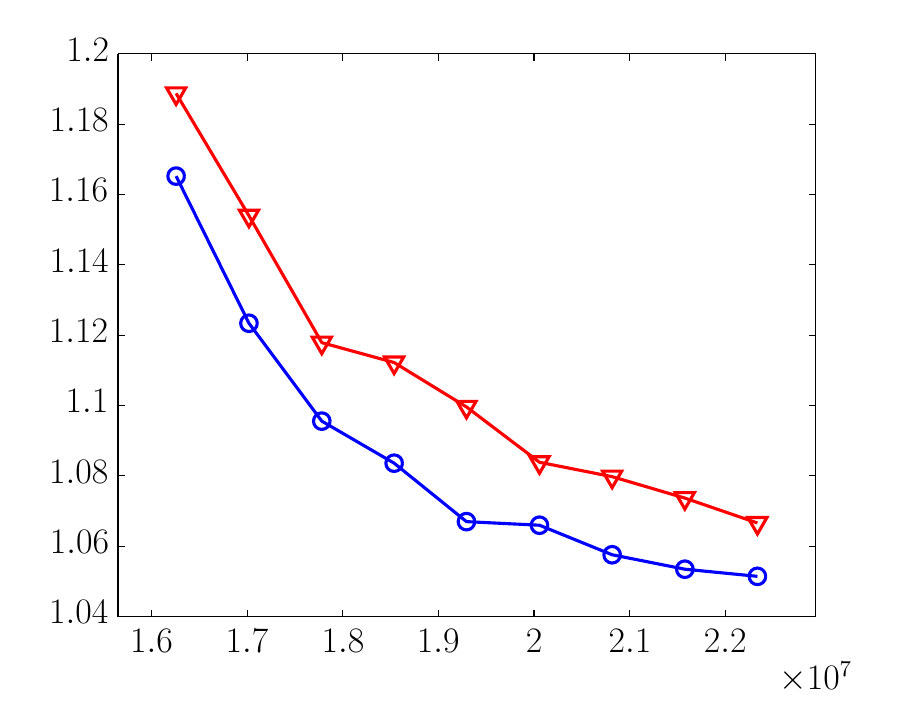}}\\
    \subfloat[grid graph, similarity-based]{\includegraphics[width=\figWidth]{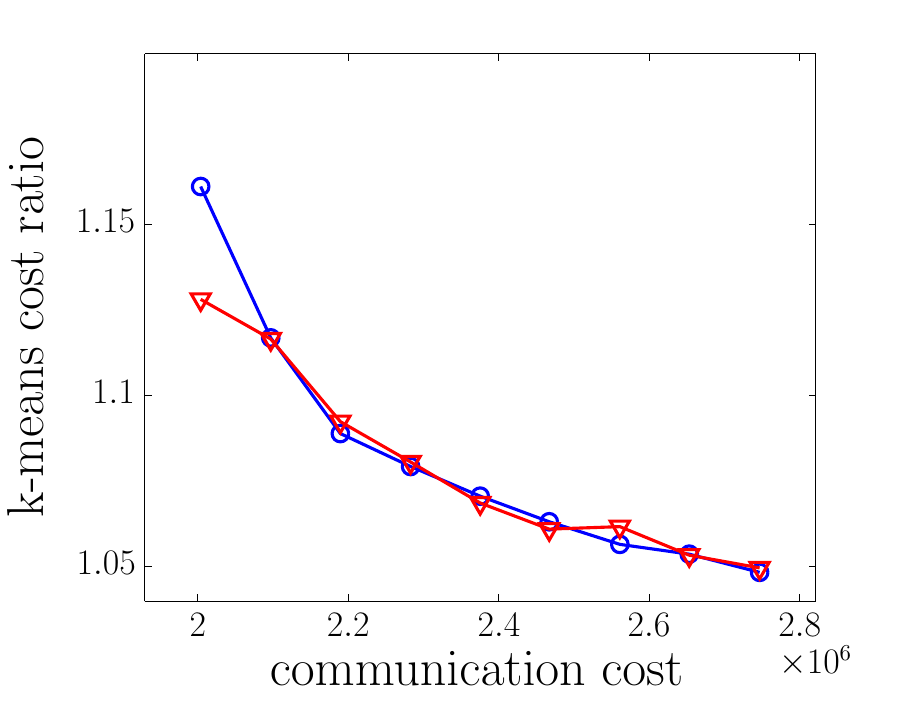}}
    \hspace*{\betweenWidth}
    \subfloat[grid graph, weighted]{\includegraphics[width=\figWidth]{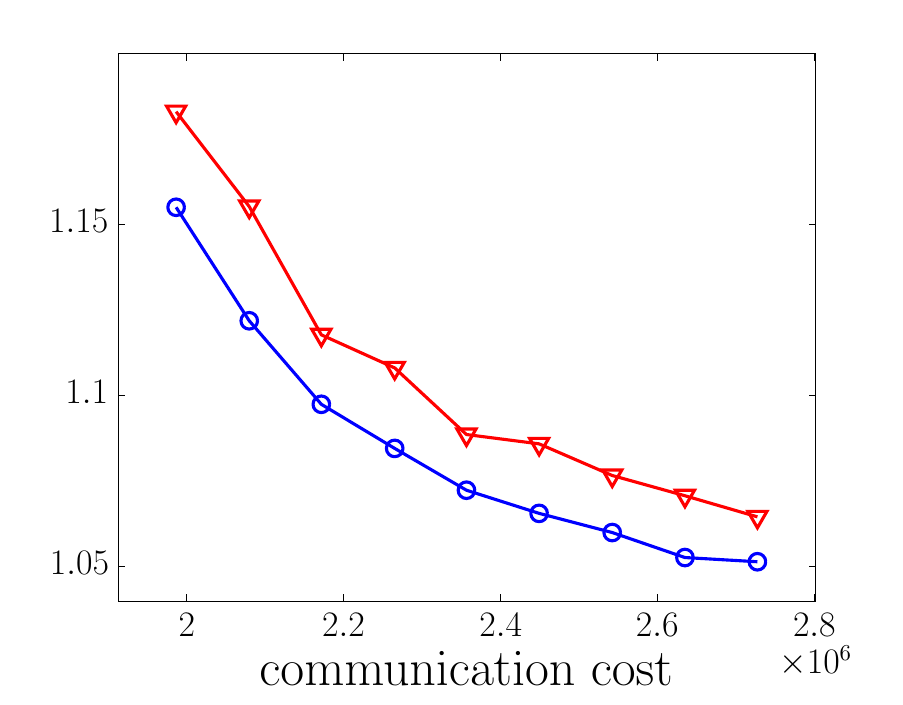}}
    \hspace*{\betweenWidth}
    \subfloat[preferential graph, degree-based]{\includegraphics[width=\figWidth]{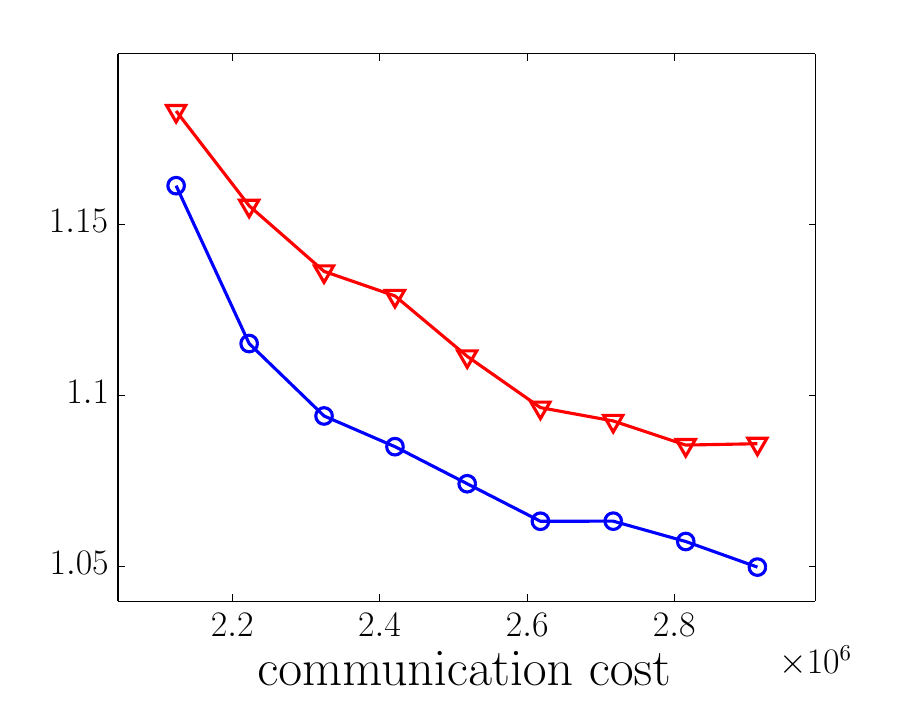}}
\end{center}
\caption{$k$-means cost (normalized by baseline) v.s. communication cost over graphs. The titles indicate the network topology and partition method.
}\label{fig:graph_result}
\end{figure}

\begin{figure}[t]
\begin{center}
\centering
    \subfloat[random graph, uniform]{\includegraphics[width=\figWidth]{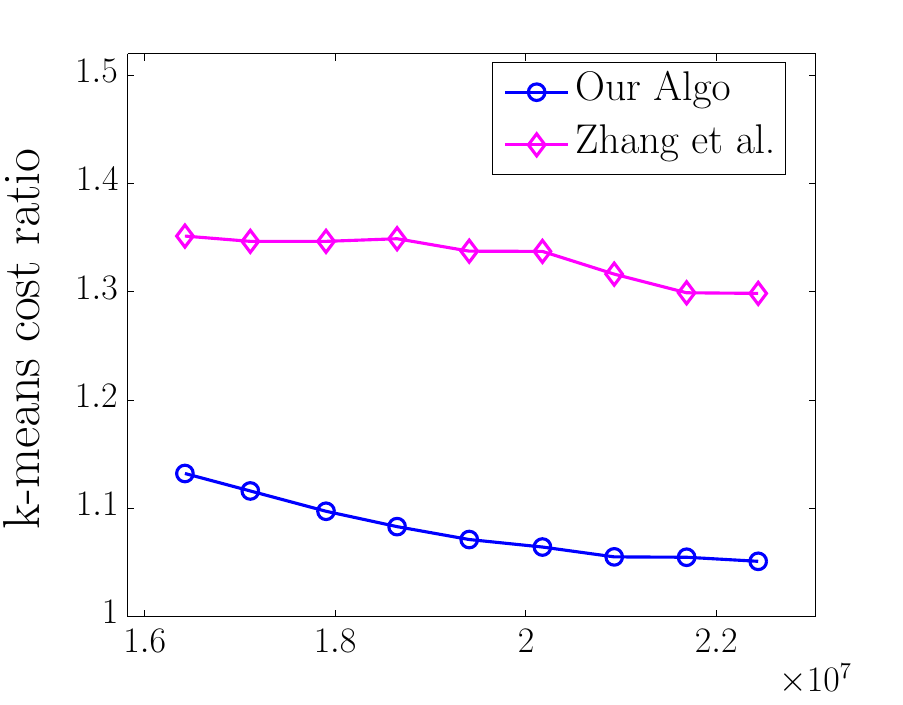}}
    \hspace*{\betweenWidth}
    \subfloat[random graph, similarity-based]{\includegraphics[width=\figWidth]{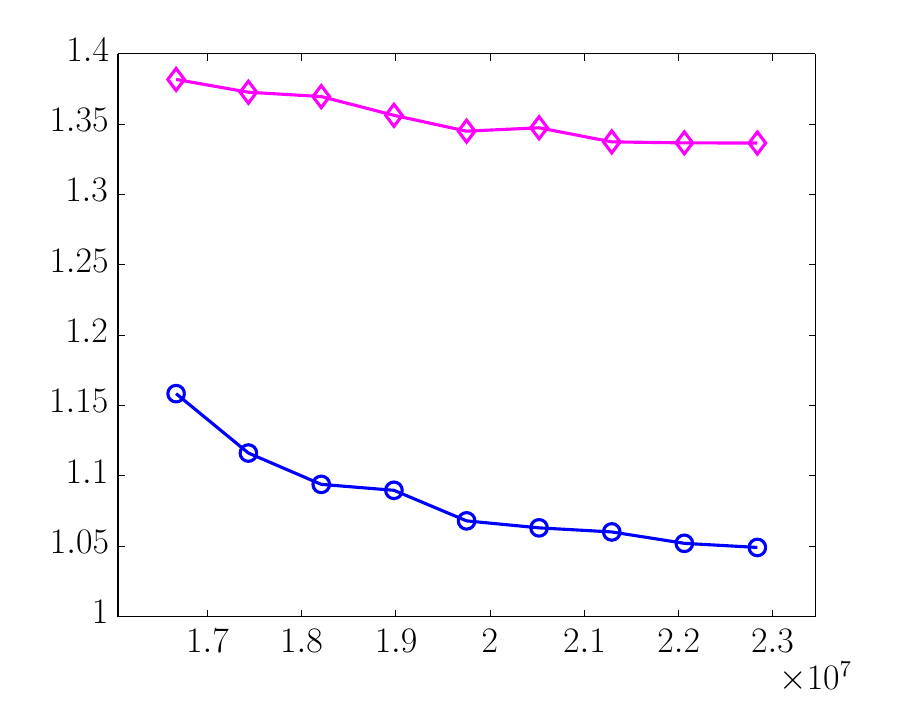}}
    \hspace*{\betweenWidth}
    \subfloat[random graph, weighted]{\includegraphics[width=\figWidth]{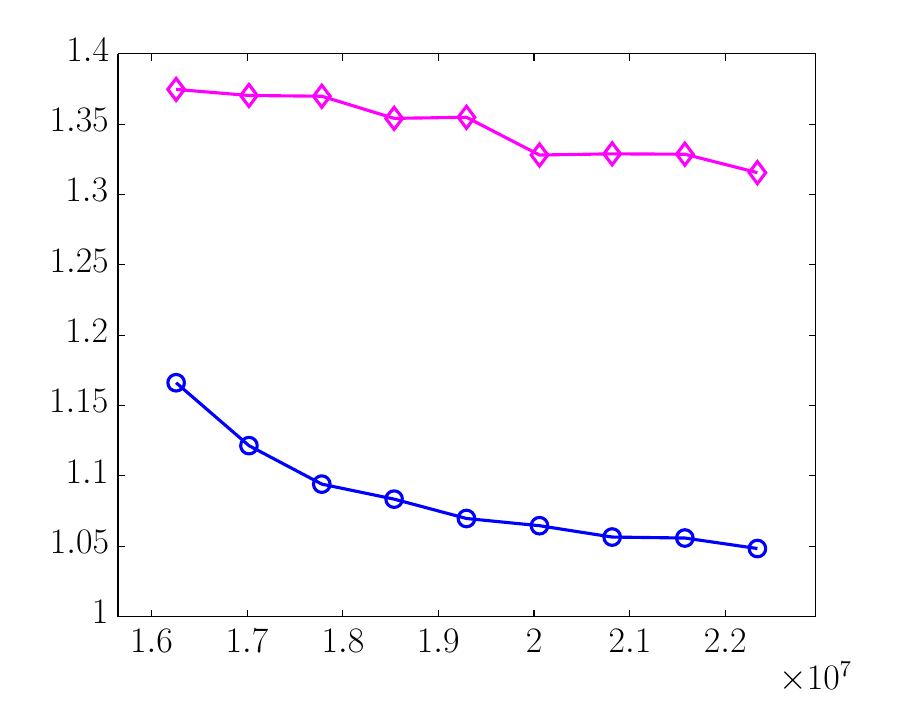}}\\
    \subfloat[grid graph, similarity-based]{\includegraphics[width=\figWidth]{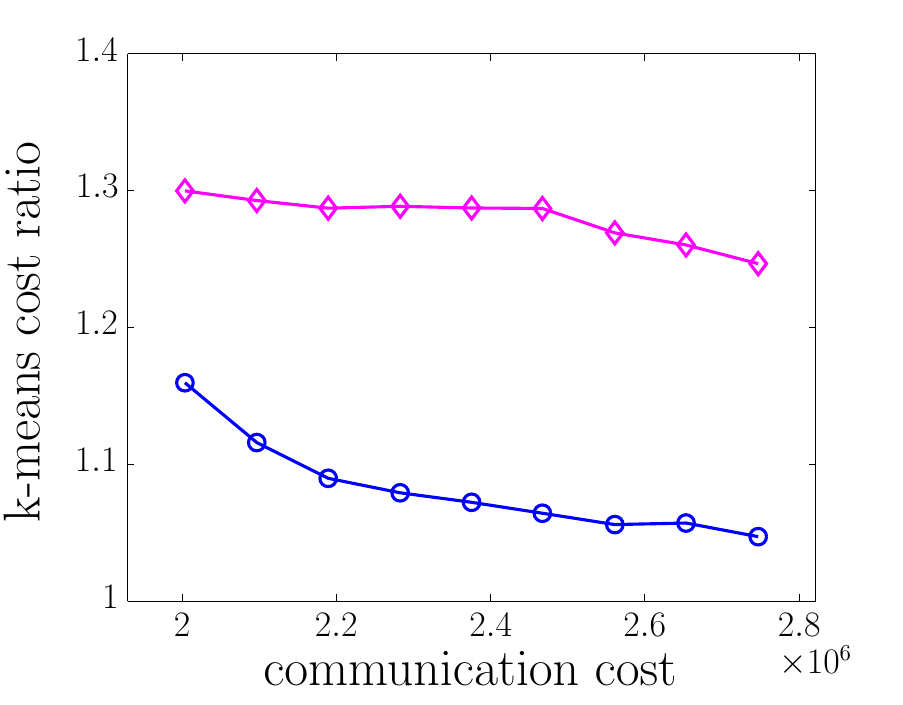}}
    \hspace*{\betweenWidth}
    \subfloat[grid graph, weighted]{\includegraphics[width=\figWidth]{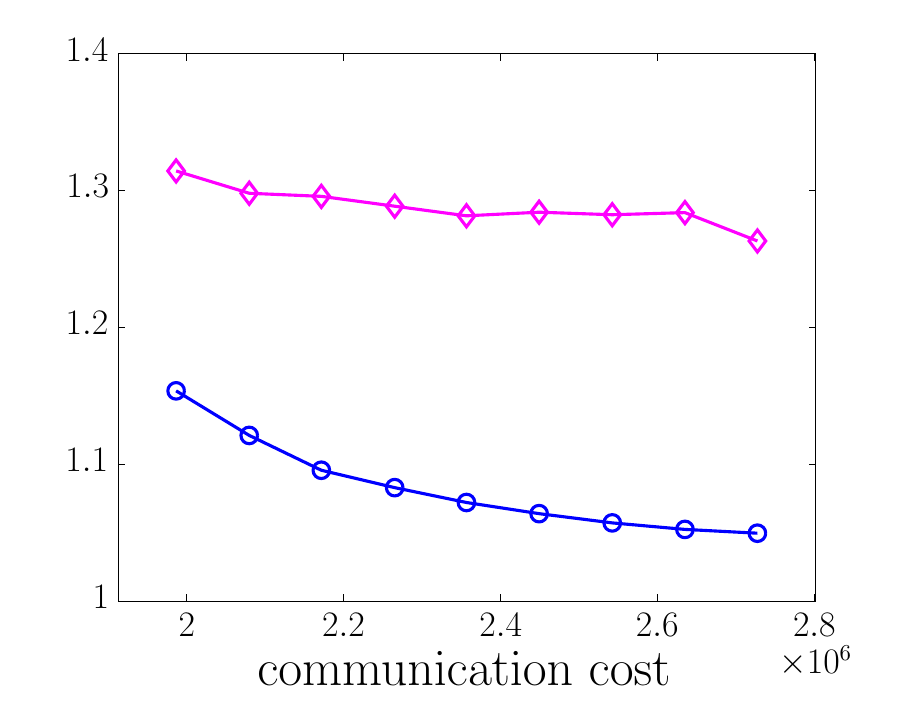}}
    \hspace*{\betweenWidth}
    \subfloat[preferential graph, degree-based]{\includegraphics[width=\figWidth]{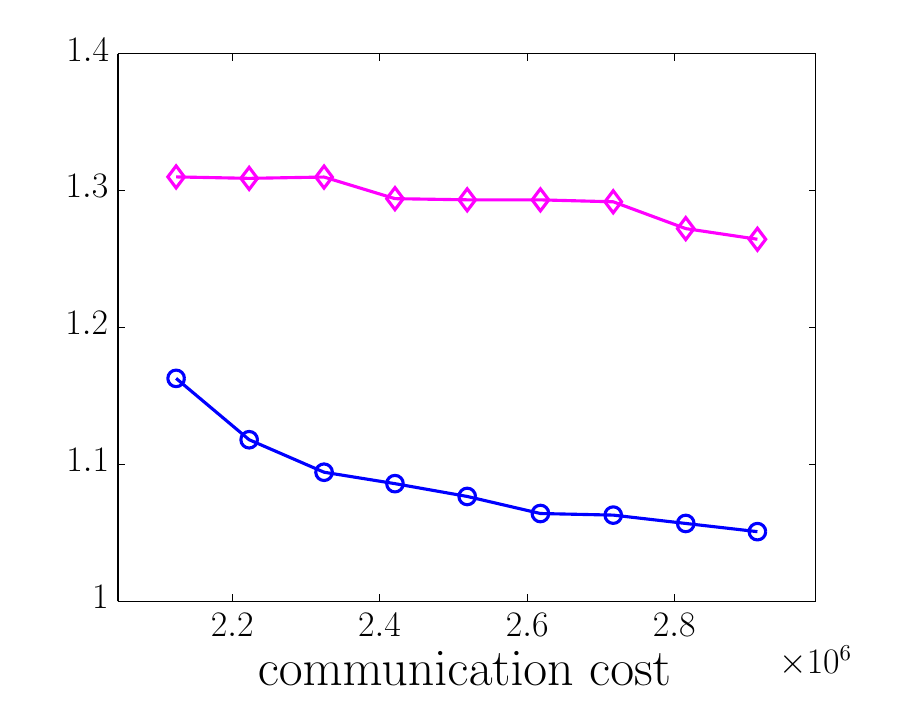}}
\end{center}
\caption{$k$-means cost (normalized by baseline) v.s. communication cost over the spanning trees of the graphs. The titles indicate the network topology and partition method.
}\label{fig:tree_result}
\end{figure}

\smallskip
\noindent
\littleheader{Results}
Here we focus on the results of the largest data set YearPredictionMSD,
and in Appendix~\ref{app:exp} we present the experimental results for all the data sets.

Figure~\ref{fig:graph_result} shows the results over different network topologies and partition methods.
We observe that the algorithms perform well with much smaller coreset sizes than predicted by the theoretical bounds.
For example, to get $1.1$ cost ratio, the coreset size and thus the communication needed is only $0.1\%-1\%$ of the theoretical bound.

In the uniform partition, our algorithm performs nearly the same as COMBINE.
This is not surprising since our algorithm reduces to the COMBINE algorithm when each local site has the same cost
and the two algorithms use the same amount of communication.
In this case, since in our algorithm the sizes of the local samples are proportional to the costs of the local solutions,
it samples the same number of points from each local data set.
This is equivalent to the COMBINE algorithm with the same amount of communication.
In the similarity-based partition, similar results are observed as it also leads to balanced local costs.
However, when the local sites have significantly different costs (as in the weighted and degree-based partitions),
our algorithm outperforms COMBINE.
As observed in Figure~\ref{fig:graph_result}, the costs of our solutions consistently improve over those of COMBINE by $2\%-5\%$.
Our algorithm then saves $10\%-20\%$ communication cost to achieve the same approximation ratio.


Figure~\ref{fig:tree_result} shows the results over the spanning trees of the graphs.
Our algorithm performs much better than the algorithm of Zhang et al., achieving about $20\%$ improvement in cost.
This is due to the fact that their algorithm needs larger coresets to prevent the accumulation of errors
when constructing coresets from component coresets, and thus needs higher communication cost to achieve the same approximation ratio.

Similar results are observed on the other datasets, which are presented in Appendix~\ref{app:exp}.

\section{Additional Related Work}\label{sec:work}

Many empirical algorithms adapt the centralized algorithms to the distributed setting.
They generally provide no bound for the clustering quality or the communication cost.
For instance, a technique is proposed in~\cite{forman2000distributed} to adapt several iterative center-based data clustering algorithms
including Lloyd's algorithm for $k$-means to the distributed setting, where sufficient statistics instead of the raw data are sent to a central coordinator.
This approach involves transferring data back and forth in each iteration, and thus the communication cost depends on the number of iterations.
Similarly, the communication costs of the distributed clustering algorithms proposed in~\cite{datta2005k} and~\cite{tasoulis2004unsupervised} depend on the number of iterations.
Some other algorithms gather local summaries and then perform global clustering on the summaries.
The distributed density-based clustering algorithm in~\cite{januzaj2003towards} clusters and
computes summaries for the local data at each node, and sends the local summaries to a central node where the global clustering is carried out.
This algorithm only considers the flat two-tier topology.
Some in-network aggregation schemes for computing statistics over distributed data are useful for such distributed clustering algorithms.
For example, an algorithm is provided in~\cite{considine2004approximate} for approximate duplicate-sensitive aggregates across distributed data sets, such as SUM. An algorithm is proposed in~\cite{greenwald2004power} for power-preserving computation of order statistics such as quantile.

Several coreset construction algorithms have been proposed for $k$-median, $k$-means and $k$-line median clustering~\cite{har2004coresets,chen2006k,har2007smaller,langberg2010universal,feldman2011unified}.
For example, the algorithm in~\cite{feldman2011unified} constructs a coreset of size $\tilde{O}(kd/\epsilon^2)$ whose cost approximates that of the original data up to accuracy $\epsilon$ with respect to $k$-median in $\R^d$.
All of these algorithms consider coreset construction in the centralized setting, while our construction algorithm is for the distributed setting.

There has also been work attempting to parallelize clustering algorithms.
\cite{feldman2012effective} showed that coresets could be constructed in parallel and then merged together.
Bahmani et al.~\cite{bahmani2012scalable} adapted {\tt k-means++} to the parallel setting.
Their algorithm, k-means$||$, essentially builds $O(1)$-coreset of size $O(k \log |P|)$.
However, it cannot build $\epsilon$-coreset for $\epsilon=o(1)$,
and thus can only guarantee constant approximation solutions.


There is also related work providing approximation solutions for $k$-median based on random sampling~\cite{ben2004framework}.
Particularly, they showed that given a sample of size $\tilde{O}(\frac{k}{\epsilon^2})$ drawn i.i.d.\ from the data,
there exists an algorithm that outputs a solution with an average cost bounded by twice the optimal average cost plus an error bound~$\epsilon$.
If we convert it to a multiplicative approximation factor, the factor depends on the optimal average cost.
When there are outlier points far away from all other points,
the optimal average cost can be very small after normalization,
then the multiplicative approximation factor is large.
The coreset approach provides better guarantees.
Additionally, their approach is not applicable to $k$-means.

Balcan et al.~\cite{balcan2012distributed} and Daume et al.~\cite{daume2012efficient} consider fundamental communication complexity questions arising when doing classification in distributed settings.
In concurrent and independent work, Vempala et al.~\cite{kannan2013nimble} study several optimization problems in distributed settings, including $k$-means clustering under an interesting separability assumption.

%

\paragraph{Acknowledgements}
This work was supported by ONR grant N00014-09-1-0751, AFOSR grant FA9550-09-1-0538, and by a Google Research Award.
We thank Le Song for generously allowing us to use his computer cluster.

\bibliographystyle{abbrv}
\bibliography{distributedClusteringRef}

\begin{thebibliography}{10}

\bibitem{Reka:RevModPhys2002}
R.~Albert and A.-L. Barab{\'a}si.
\newblock Statistical mechanics of complex networks.
\newblock {\em Reviews of Modern Physics}, 2002.

\bibitem{ab:survey}
P.~Awasthi and M.~Balcan.
\newblock Center based clustering: A foundational perspective.
\newblock Survey Chapter in Handbook of Cluster Analysis (Manuscript), 2013.

\bibitem{Bache+Lichman:2013}
K.~Bache and M.~Lichman.
\newblock {UCI} machine learning repository, 2013.

\bibitem{bachem2018scalable}
O.~Bachem, M.~Lucic, and A.~Krause.
\newblock Scalable and distributed clustering via lightweight coresets.
\newblock In {\em ACM SIGKDD International Conference on Knowledge Discovery
  and Data Mining (KDD)}, 2018.

\bibitem{bahmani2012scalable}
B.~Bahmani, B.~Moseley, A.~Vattani, R.~Kumar, and S.~Vassilvitskii.
\newblock Scalable k-means++.
\newblock In {\em Proceedings of the International Conference on Very Large
  Data Bases}, 2012.

\bibitem{balcan2012distributed}
M.-F. Balcan, A.~Blum, S.~Fine, and Y.~Mansour.
\newblock Distributed learning, communication complexity and privacy.
\newblock In {\em Proceedings of the Conference on Learning Thoery}, 2012.

\bibitem{ben2004framework}
S.~Ben-David.
\newblock A framework for statistical clustering with a constant time
  approximation algorithms for k-median clustering.
\newblock {\em Proceedings of Annual Conference on Learning Theory}, 2004.

\bibitem{chen2006k}
K.~Chen.
\newblock On k-median clustering in high dimensions.
\newblock In {\em Proceedings of the Annual ACM-SIAM Symposium on Discrete
  Algorithms}, 2006.

\bibitem{considine2004approximate}
J.~Considine, F.~Li, G.~Kollios, and J.~Byers.
\newblock Approximate aggregation techniques for sensor databases.
\newblock In {\em Proceedings of the International Conference on Data
  Engineering}, 2004.

\bibitem{corbett2012spanner}
J.~C. Corbett, J.~Dean, M.~Epstein, A.~Fikes, C.~Frost, J.~Furman, S.~Ghemawat,
  A.~Gubarev, C.~Heiser, P.~Hochschild, et~al.
\newblock Spanner: Google��s globally-distributed database.
\newblock In {\em Proceedings of the USENIX Symposium on Operating Systems
  Design and Implementation}, 2012.

\bibitem{datta2005k}
S.~Datta, C.~Giannella, H.~Kargupta, et~al.
\newblock K-means clustering over peer-to-peer networks.
\newblock In {\em Proceedings of the International Workshop on High Performance
  and Distributed Mining}, 2005.

\bibitem{daume2012efficient}
H.~Daum{\'e}~III, J.~M. Phillips, A.~Saha, and S.~Venkatasubramanian.
\newblock Efficient protocols for distributed classification and optimization.
\newblock In {\em Algorithmic Learning Theory}, pages 154--168. Springer, 2012.

\bibitem{feldman2011unified}
D.~Feldman and M.~Langberg.
\newblock A unified framework for approximating and clustering data.
\newblock In {\em Proceedings of the Annual ACM Symposium on Theory of
  Computing}, 2011.

\bibitem{feldman2012effective}
D.~Feldman, A.~Sugaya, and D.~Rus.
\newblock An effective coreset compression algorithm for large scale sensor
  networks.
\newblock In {\em Proceedings of the International Conference on Information
  Processing in Sensor Networks}, 2012.

\bibitem{forman2000distributed}
G.~Forman and B.~Zhang.
\newblock Distributed data clustering can be efficient and exact.
\newblock {\em ACM SIGKDD Explorations Newsletter}, 2000.

\bibitem{greenhill2007distributed}
S.~Greenhill and S.~Venkatesh.
\newblock Distributed query processing for mobile surveillance.
\newblock In {\em Proceedings of the International Conference on Multimedia},
  2007.

\bibitem{greenwald2004power}
M.~Greenwald and S.~Khanna.
\newblock Power-conserving computation of order-statistics over sensor
  networks.
\newblock In {\em Proceedings of the ACM SIGMOD-SIGACT-SIGART Symposium on
  Principles of Database Systems}, 2004.

\bibitem{har2011geometric}
S.~Har-Peled.
\newblock {\em Geometric approximation algorithms}.
\newblock Number Vol. 173. American Mathematical Society, 2011.

\bibitem{har2007smaller}
S.~Har-Peled and A.~Kushal.
\newblock Smaller coresets for k-median and k-means clustering.
\newblock {\em Discrete \& Computational Geometry}, 2007.

\bibitem{har2004coresets}
S.~Har-Peled and S.~Mazumdar.
\newblock On coresets for k-means and k-median clustering.
\newblock In {\em Proceedings of the Annual ACM Symposium on Theory of
  Computing}, 2004.

\bibitem{januzaj2003towards}
E.~Januzaj, H.~Kriegel, and M.~Pfeifle.
\newblock Towards effective and efficient distributed clustering.
\newblock In {\em Workshop on Clustering Large Data Sets in the IEEE
  International Conference on Data Mining}, 2003.

\bibitem{kannan2013nimble}
R.~Kannan and S.~Vempala.
\newblock Nimble algorithms for cloud computing.
\newblock {\em arXiv preprint arXiv:1304.3162}, 2013.

\bibitem{kanungo2002local}
T.~Kanungo, D.~M. Mount, N.~S. Netanyahu, C.~D. Piatko, R.~Silverman, and A.~Y.
  Wu.
\newblock A local search approximation algorithm for k-means clustering.
\newblock In {\em Proceedings of the Annual Symposium on Computational
  Geometry}, 2002.

\bibitem{kargupta2001distributed}
H.~Kargupta, W.~Huang, K.~Sivakumar, and E.~Johnson.
\newblock Distributed clustering using collective principal component analysis.
\newblock {\em Knowledge and Information Systems}, 2001.

\bibitem{langberg2010universal}
M.~Langberg and L.~Schulman.
\newblock Universal $\varepsilon$-approximators for integrals.
\newblock In {\em Proceedings of the Annual ACM-SIAM Symposium on Discrete
  Algorithms}, 2010.

\bibitem{shi2013app}
S.~Li and O.~Svensson.
\newblock Approximating k-median via pseudo-approximation.
\newblock In {\em Proceedings of the Annual ACM Symposium on Theory of
  Computing}, 2013.

\bibitem{li2000improved}
Y.~Li, P.~M. Long, and A.~Srinivasan.
\newblock Improved bounds on the sample complexity of learning.
\newblock In {\em Proceedings of the eleventh annual ACM-SIAM Symposium on
  Discrete Algorithms}, 2000.

\bibitem{mitra2011characterizing}
S.~Mitra, M.~Agrawal, A.~Yadav, N.~Carlsson, D.~Eager, and A.~Mahanti.
\newblock Characterizing web-based video sharing workloads.
\newblock {\em ACM Transactions on the Web}, 2011.

\bibitem{olston2003adaptive}
C.~Olston, J.~Jiang, and J.~Widom.
\newblock Adaptive filters for continuous queries over distributed data
  streams.
\newblock In {\em Proceedings of the ACM SIGMOD International Conference on
  Management of Data}, 2003.

\bibitem{tasoulis2004unsupervised}
D.~Tasoulis and M.~Vrahatis.
\newblock Unsupervised distributed clustering.
\newblock In {\em Proceedings of the International Conference on Parallel and
  Distributed Computing and Networks}, 2004.

\bibitem{zhang2008approximate}
Q.~Zhang, J.~Liu, and W.~Wang.
\newblock Approximate clustering on distributed data streams.
\newblock In {\em Proceedings of the IEEE International Conference on Data
  Engineering}, 2008.

\end{thebibliography}

\appendix

\section{Proofs for Section \ref{sec:coreset}}
The proof of Lemma~\ref{lem:sampling} follows from the analysis in~\cite{feldman2011unified}, although not explicitly stated there.
We begin with the following theorem for uniform sampling on a function space.
The theorem is from~\cite{feldman2011unified} but rephrased for convenience (and corrected).\footnote{As pointed out in~\cite{bachem2018scalable}, the proof in~\cite{feldman2011unified} used a theorem from~\cite{li2000improved} which used the notion of pseudo-dimension $d$ of a function space. However, the definition of the dimension $d'$ of a function space in~\cite{feldman2011unified} is different from $d$. Fortunately, by~\cite{har2011geometric}, $d = O(d' \log d')$. Therefore, the theorem is corrected by replacing $\dim(F, P)$ with $\dim(F, P) \log \dim(F, P)$.}

\begin{theorem}[Theorem 6.9 in~\cite{feldman2011unified}]\label{thm:sampling}
Let $F$ be a set of functions from $P$ to $\R_{\geq 0}$, and let $\epsilon \in (0, 1)$. Let $S$ be a sample of
$$ |S| = \frac{c}{\epsilon^2}(\dim(F, P) \log \dim(F, P) + \log\frac{1}{\delta})$$
i.i.d items from $P$, where $c$ is a sufficiently large constant. Then, with probability at least $1 - \delta$,
for any $f \in F$ and any $r \geq 0$,
\begin{eqnarray*}
\left|\frac{\sum_{p \in P, f(p) \leq r} f(p)}{|P|} - \frac{\sum_{q \in S, f(q) \leq r} f(q)}{|S|}\right| \leq \epsilon r.
\end{eqnarray*}
\end{theorem}

\noindent
\textbf{Lemma~\ref{lem:sampling}} (Restated)\textbf{.}
\textit{
Fix a set $F$ of functions $f:P \to \R_{\geq 0}$.
Let $S$ be a sample drawn i.i.d. from $P$ according to $\{m_p: p\in P\}$, namely, for every $q \in S$ and every $p \in P$, we have $q=p$ with probability $\frac{m_p}{\sum_{z \in P}m_z}$.
Let $w_p=\frac {\sum_{z\in P}m_z}{m_p|S|}$ for $p \in P$.
For a sufficiently large $c$,
if $|S| \geq \frac c {\epsilon^2} \left(\dim(F,P) \log \dim(F,P) +\log \frac 1\delta\right)$
then with probability at least $1-\delta, \forall f \in F:$
$\left| \sum_{p\in P} f(p) - \sum_{q\in S} w_q f(q)\right| \leq \epsilon \left(\sum_{p\in P}m_p \right)\left(\max_{p\in P} \frac{f(p)}{m_p}\right).$
}

\begin{proof}[Proof of Lemma~\ref{lem:sampling}]
Without loss of generality, assume $m_p \in \mathbf{N}^+$.
Define $G$ as follows: for each $p \in P$, include $m_p$ copies $\{p_i\}_{i=1}^{m_p}$ of $p$ in $G$
and define $f(p_i) = f(p)/m_p$.
Then $S$ is equivalent to a sample draw i.i.d.\ and uniformly at random from $G$.
We now apply Theorem~\ref{thm:sampling} on $G$ and $r = \max_{f \in F, p' \in G} f(p')$.
By Theorem~\ref{thm:sampling}, we know that for any $f \in F$,
\begin{eqnarray}
\left|\frac{\sum_{p' \in G} f(p')}{|G|} - \frac{\sum_{q' \in S} f(q')}{|S|}\right| \leq \epsilon \max_{p' \in G} f(p').\label{eqn:sampling}
\end{eqnarray}
The lemma then follows from multiplying both sides of~(\ref{eqn:sampling}) by $|G| = \sum_{p \in P} m_p$.
Also note that the dimension $\dim(F,G)$ is the same as that of $\dim(F,P)$ as pointed out by~\cite{feldman2011unified}.
\end{proof}

\begin{lemma}\label{lem:techBound}
If $d(p,b_p)^2/\epsilon \leq |d(p,\x)^2 - d(b_p,\x)^2|$, then
$$|d(p,\x)^2 - d(b_p,\x)^2| \leq 8\epsilon \min\{d(p,\x)^2, d(b_p,\x)^2\}.$$
\end{lemma}
\begin{proof}
We first have by triangle inequality
\begin{eqnarray*}
|d(p,\x)^2 - d(b_p,\x)^2| \leq  d(p,b_p)[d(p,\x) + d(b_p,\x)].
\end{eqnarray*}
Then by $d(p,b_p)^2/\epsilon \leq |d(p,\x)^2 - d(b_p,\x)^2|$,
\begin{eqnarray*}
d(p,b_p) \leq  \epsilon[d(p,\x) + d(b_p,\x)].
\end{eqnarray*}
Therefore, we have
\begin{eqnarray*}
|d(p,\x)^2 - d(b_p,\x)^2| & \leq  & d(p,b_p)[d(p,\x) + d(b_p,\x)] \leq \epsilon[d(p,\x) + d(b_p,\x)]^2 \\
& \leq & 2\epsilon[d(p,\x)^2 + d(b_p,\x)^2] \leq 2\epsilon[d(p,\x)^2 + (d(p,\x) + d(p,b_p))^2] \\
& \leq & 2\epsilon[d(p,\x)^2 + 2d(p,\x)^2 + 2d(p,b_p)^2] \leq 6\epsilon d(p,\x)^2 + 4\epsilon d(p,b_p)^2 \\
& \leq & 6\epsilon d(p,\x)^2 + 4\epsilon^2 |d(p,\x)^2 - d(b_p,\x)^2|
\end{eqnarray*}
for sufficiently small $\epsilon$. Then
\begin{eqnarray*}
|d(p,\x)^2 - d(b_p,\x)^2| & \leq & \frac{6\epsilon}{1-4\epsilon^2} d(p,\x)^2  \leq 8\epsilon d(p,\x)^2.
\end{eqnarray*}
Similarly, $|d(p,\x)^2 - d(b_p,\x)^2| \leq 8\epsilon d(b_p,\x)^2$.
The lemma follows from the last two inequalities.
\end{proof}

\begin{lemma}[Corollary 15.4 in~\cite{feldman2011unified}]\label{lem:weight}
Let $0 < \delta <1/2$, and
$t \geq c |B| \log\frac{|B|}{\delta}$
for a sufficiently large $c$.
Then with probability at least $1-\delta$,
$\forall b\in B_i, \sum_{q \in P_b \cap S} w_q \leq 2 |P_b|.$
\end{lemma}

\section{Complete Experimental Results}\label{app:exp}

Here we present the results of all the data sets over different network topologies and data partition methods.

Figure~\ref{fig:cost_graph_rand} shows the results of all the data sets on random graphs.
The first column of Figure~\ref{fig:cost_graph_rand} shows that our algorithm and COMBINE perform nearly the same in the uniform data partition.
This is not surprising since our algorithm reduces to the COMBINE algorithm when each local site has the same cost
and the two algorithms use the same amount of communication.
In this case, since in our algorithm the sizes of the local samples are proportional to the costs of the local solutions,
it samples the same number of points from each local data set.
This is equivalent to the COMBINE algorithm with the same amount of communication.
In the similarity-based partition, similar results are observed as this partition method also leads to balanced local costs.
However, in the weighted partition where local sites have significantly different contributions to the total cost,
our algorithm outperforms COMBINE.
It improves the $k$-means cost by $2\%-5\%$, and thus saves $10\%-30\%$ communication cost to achieve the same approximation ratio.

Figure~\ref{fig:cost_graph_grid} shows the results of all the data sets on grid and preferential graphs.
Similar to the results on random graphs, our algorithm performs nearly the same as COMBINE in the similarity-based partition
and outperforms COMBINE in the weighted partition and degree-based partition.
Furthermore, Figure~\ref{fig:cost_graph_rand} and~\ref{fig:cost_graph_grid} also show that
the performance of our algorithm merely changes over different network topologies and partition methods.


Figure~\ref{fig:cost_tree_rand} shows the results of all the data sets on the spanning trees of the random graphs
and Figure~\ref{fig:cost_tree_grid} shows those on the spanning trees of the grid and preferential graphs.
Compared to the algorithm of Zhang~et~al., our algorithm consistently shows much better performance on all the data sets in different settings.
It improves the $k$-means cost by $10\%-30\%$, and thus can achieve even better approximation ratio with only $10\%$ communication cost.
This is because the algorithm of Zhang~et~al. constructs coresets from component coresets and needs larger coresets to prevent the accumulation of errors.
Figure~\ref{fig:cost_tree_rand} also shows that although their costs decrease with the increase of the communication,
the decrease is slower on larger graphs (e.g., as in the experiments for YearPredictionMSD).
This is due to the fact that the spanning tree of a larger graph has larger height, leading to more accumulation of errors.
In this case, more communication is needed to prevent the accumulation.

\newcommand{\figScale}{0.45}
\begin{figure*}[!p]
\centering
\begin{tabular}{ccccc}
\quad \footnotesize random graph & \qquad\qquad &\quad  \footnotesize random graph  & \quad \qquad\ & \footnotesize random graph\\
 \quad \footnotesize uniform partition &  & \quad \footnotesize similarity-based   partition & &  \footnotesize weighted partition
\end{tabular}
\begin{minipage}[t]{\textwidth}
\centering
\includegraphics[scale = \figScale]{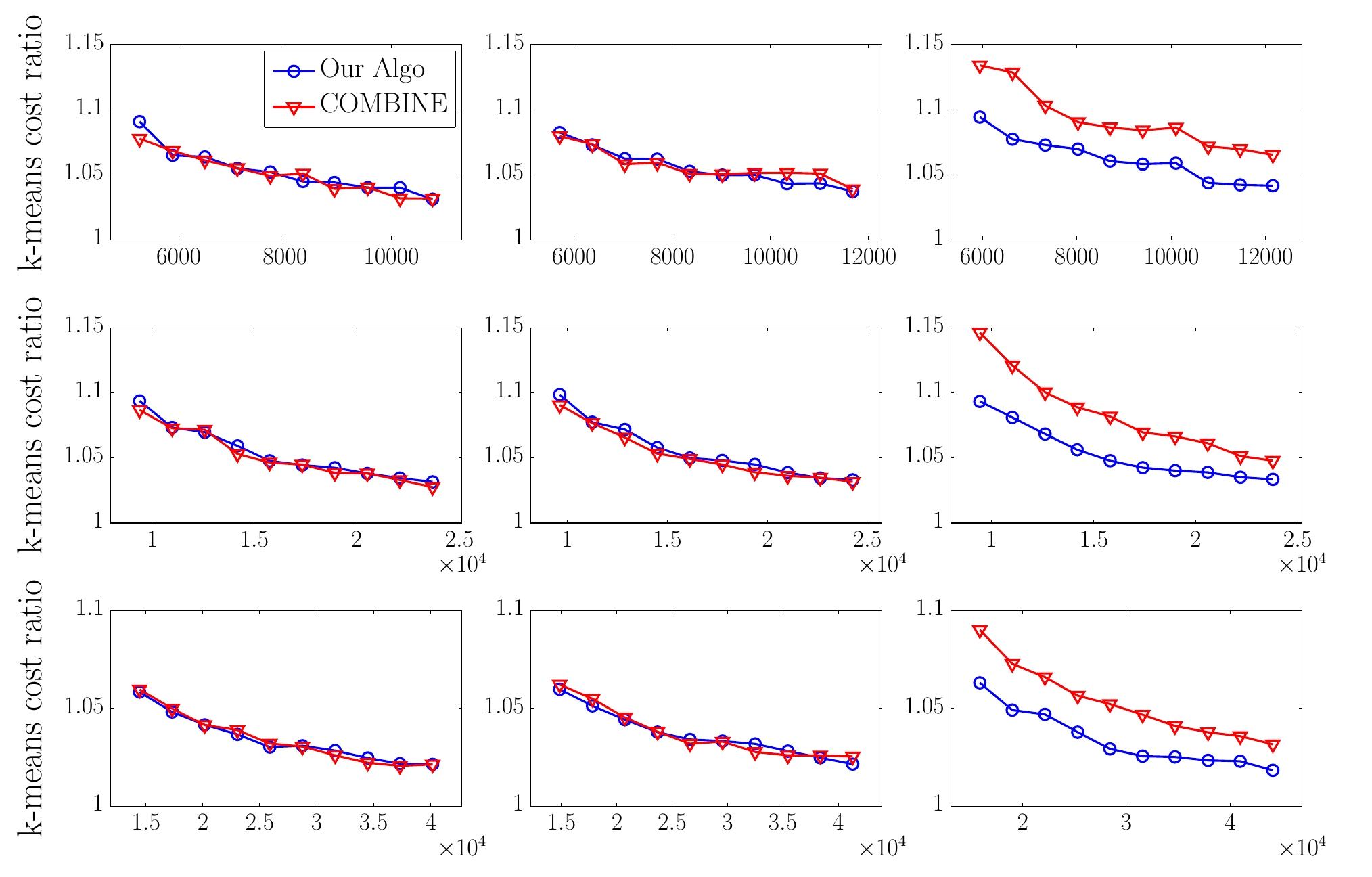}
\end{minipage}
\vspace*{-.3in}
\begin{minipage}[t]{\textwidth}
\centering
\includegraphics[scale = \figScale]{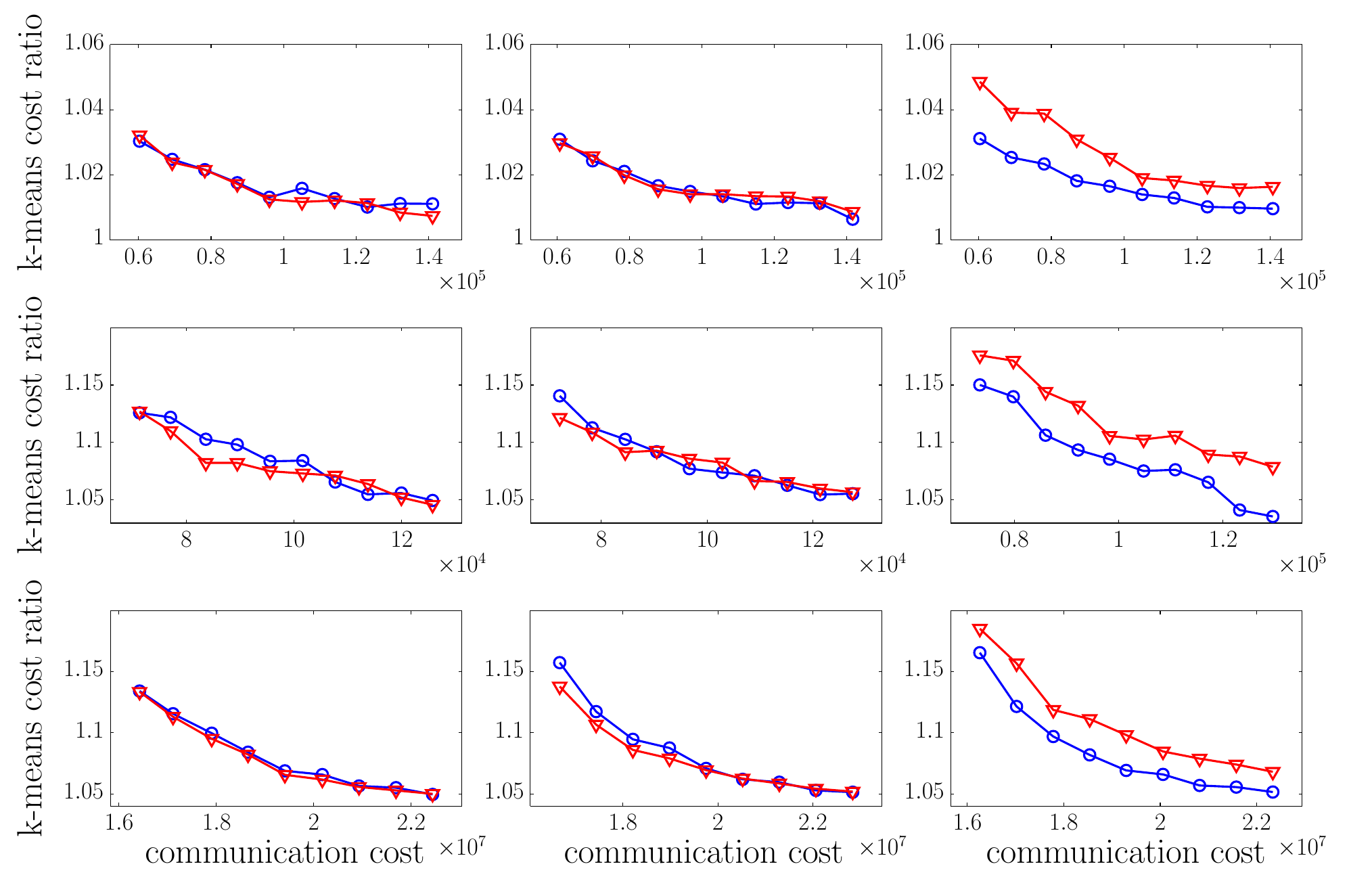}
\caption{$k$-means cost on random graphs. Columns: random graph with uniform partition, random graph with similarity-based partition, and random graph with weighted partition. Rows: Spam, Pendigits, Letter, synthetic, ColorHistogram, and YearPredictionMSD.}\label{fig:cost_graph_rand}
\end{minipage}
\end{figure*}

\begin{figure*}[!p]
\centering
\begin{tabular}{ccccc}
\quad\footnotesize grid graph & \qquad\qquad &\quad  \footnotesize grid graph  & \qquad \qquad\ & \footnotesize preferential graph\\
 \quad \footnotesize similarity-based partition &  & \quad \footnotesize  weighted partition & &  \footnotesize degree-based partition
\end{tabular}
\begin{minipage}[t]{\textwidth}
\centering
\includegraphics[scale = \figScale]{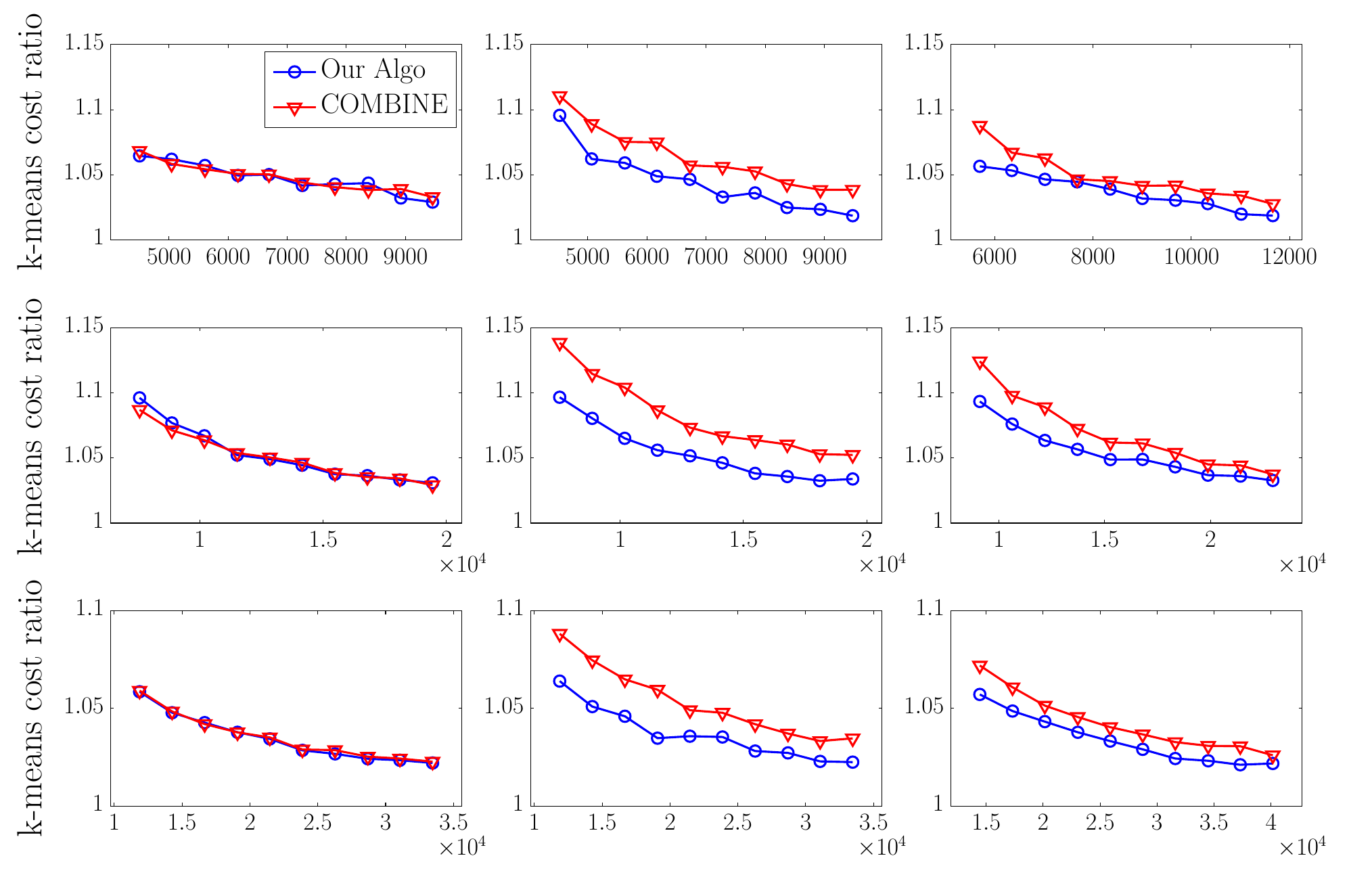}
\end{minipage}
\vspace*{-.3in}
\begin{minipage}[t]{\textwidth}
\centering
\includegraphics[scale = \figScale]{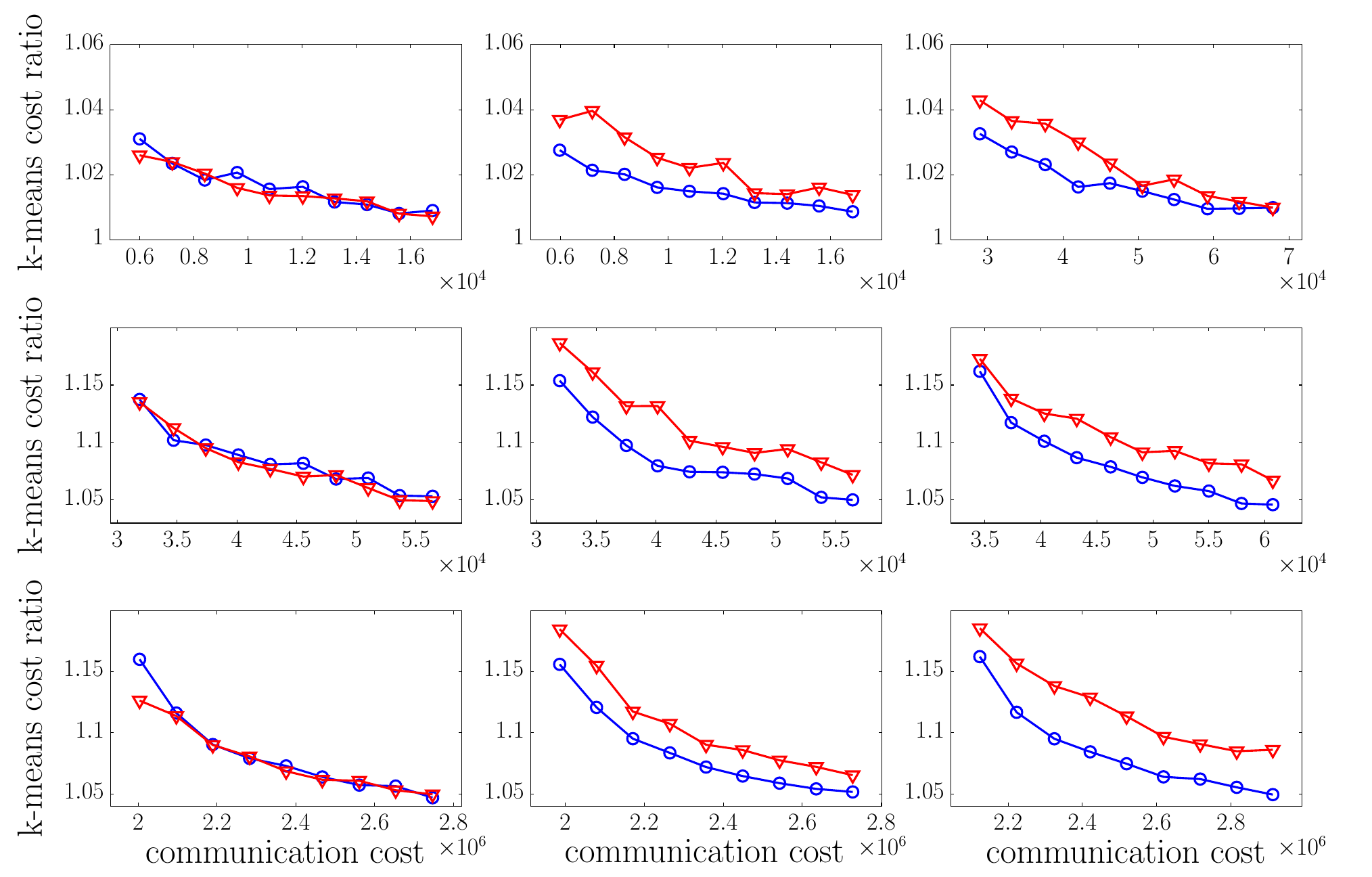}
\caption{$k$-means cost on grid and preferential graphs. Columns: grid graph with similarity-based partition, grid graph with weighted partition, and preferential graph with degree-based partition. Rows: Spam, Pendigits, Letter, synthetic, ColorHistogram, and YearPredictionMSD.}\label{fig:cost_graph_grid}
\end{minipage}
\end{figure*}

\begin{figure*}[!p]
\centering
\begin{tabular}{ccccc}
\quad \footnotesize spanning tree of random graph & & \footnotesize spanning tree of  random graph  & & \footnotesize spanning tree of random graph\\
\quad \footnotesize  uniform partition &   &  \footnotesize similarity-based partition &  &  \footnotesize weighted partition
\end{tabular}
\begin{minipage}[t]{\textwidth}
\centering
\includegraphics[scale = \figScale]{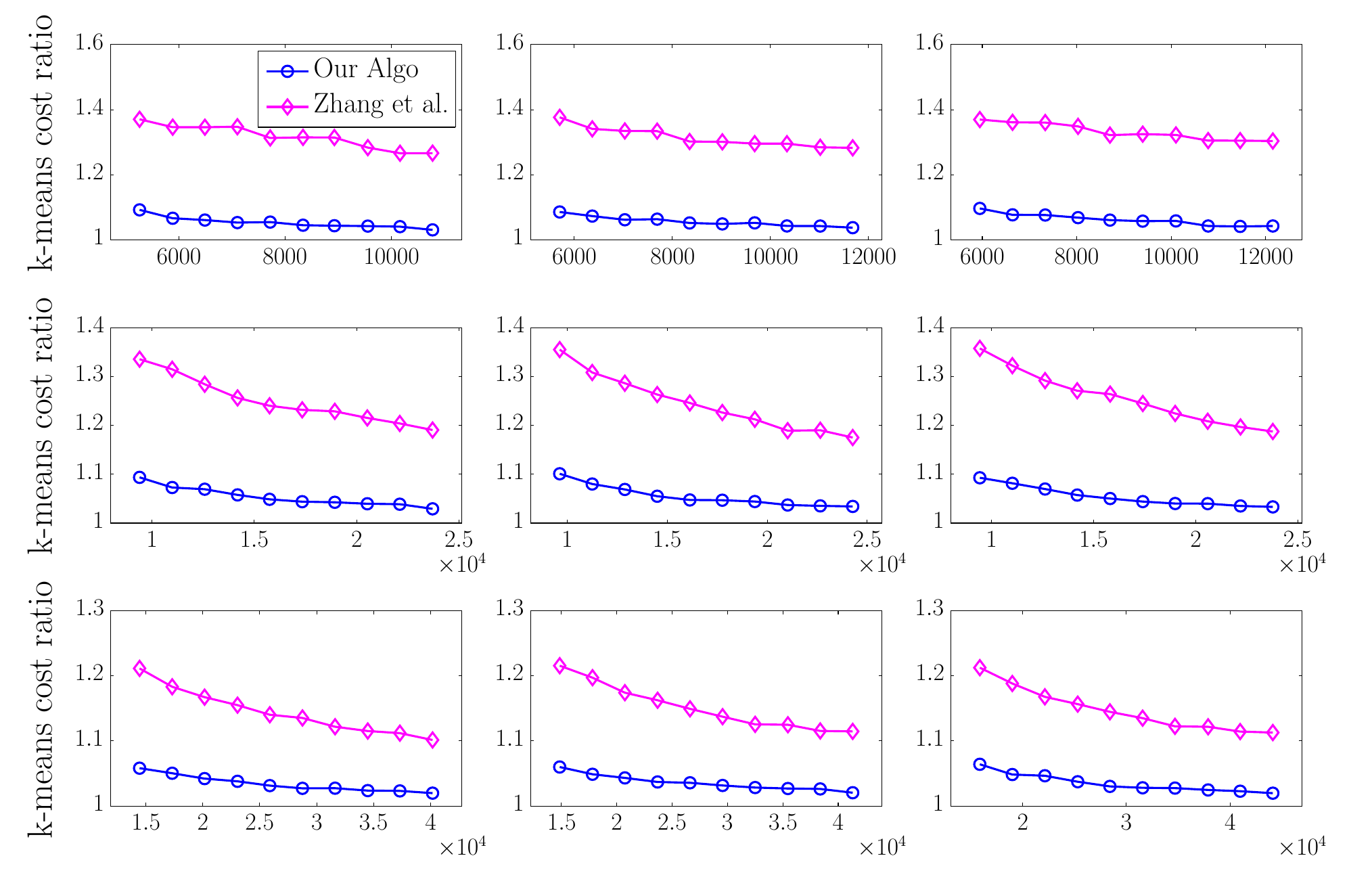}
\end{minipage}
\vspace*{-.3in}
\begin{minipage}[t]{\textwidth}
\centering
\includegraphics[scale = \figScale]{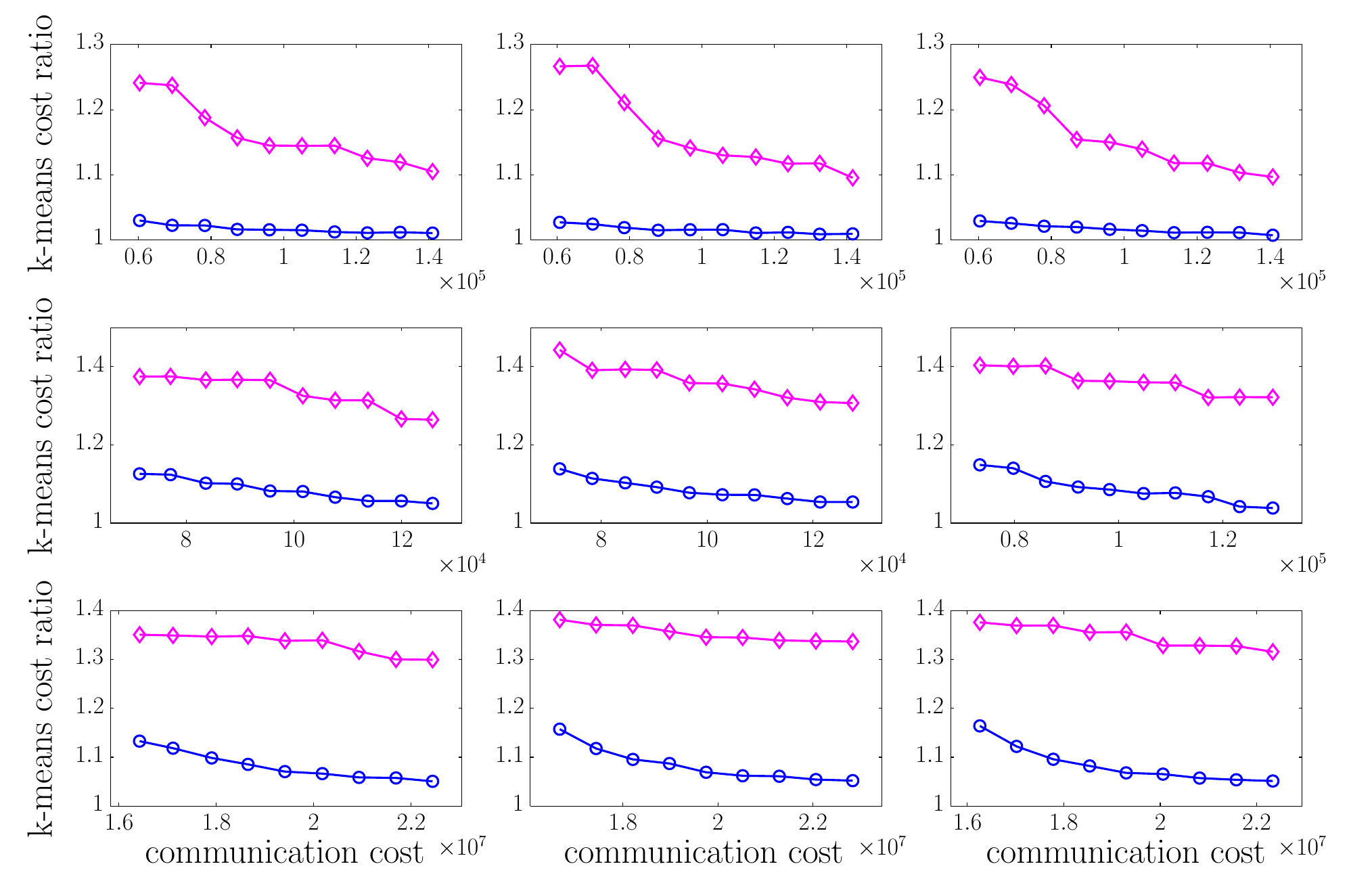}
\caption{$k$-means cost on the spanning trees of the random graphs. Columns: random graph with uniform partition, random graph with similarity-based partition, and random graph with weighted partition. Rows: Spam, Pendigits, Letter, synthetic, ColorHistogram, and YearPredictionMSD.}\label{fig:cost_tree_rand}
\end{minipage}
\end{figure*}

\begin{figure*}[!p]
\centering
\begin{tabular}{ccccc}
\qquad \footnotesize spanning tree of grid graph & & \quad \footnotesize spanning tree of  grid graph  & & \footnotesize spanning tree of preferential graph\\
\qquad \footnotesize  similarity-based partition &   &  \quad \footnotesize  weighted partition &  &  \footnotesize degree-based partition
\end{tabular}
\begin{minipage}[t]{\textwidth}
\centering
\includegraphics[scale = \figScale]{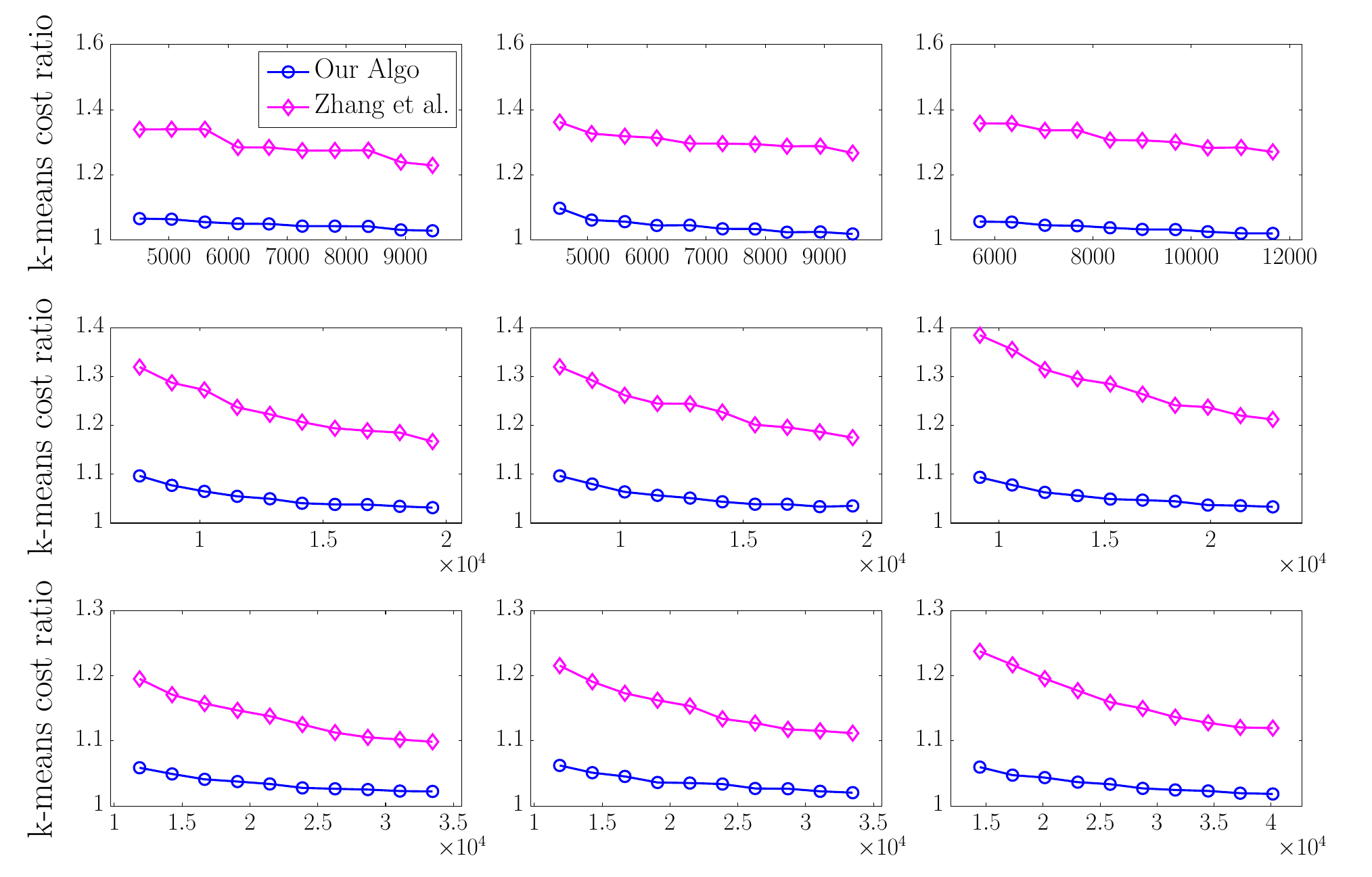}
\end{minipage}
\vspace*{-.3in}
\begin{minipage}[t]{\textwidth}
\centering
\includegraphics[scale = \figScale]{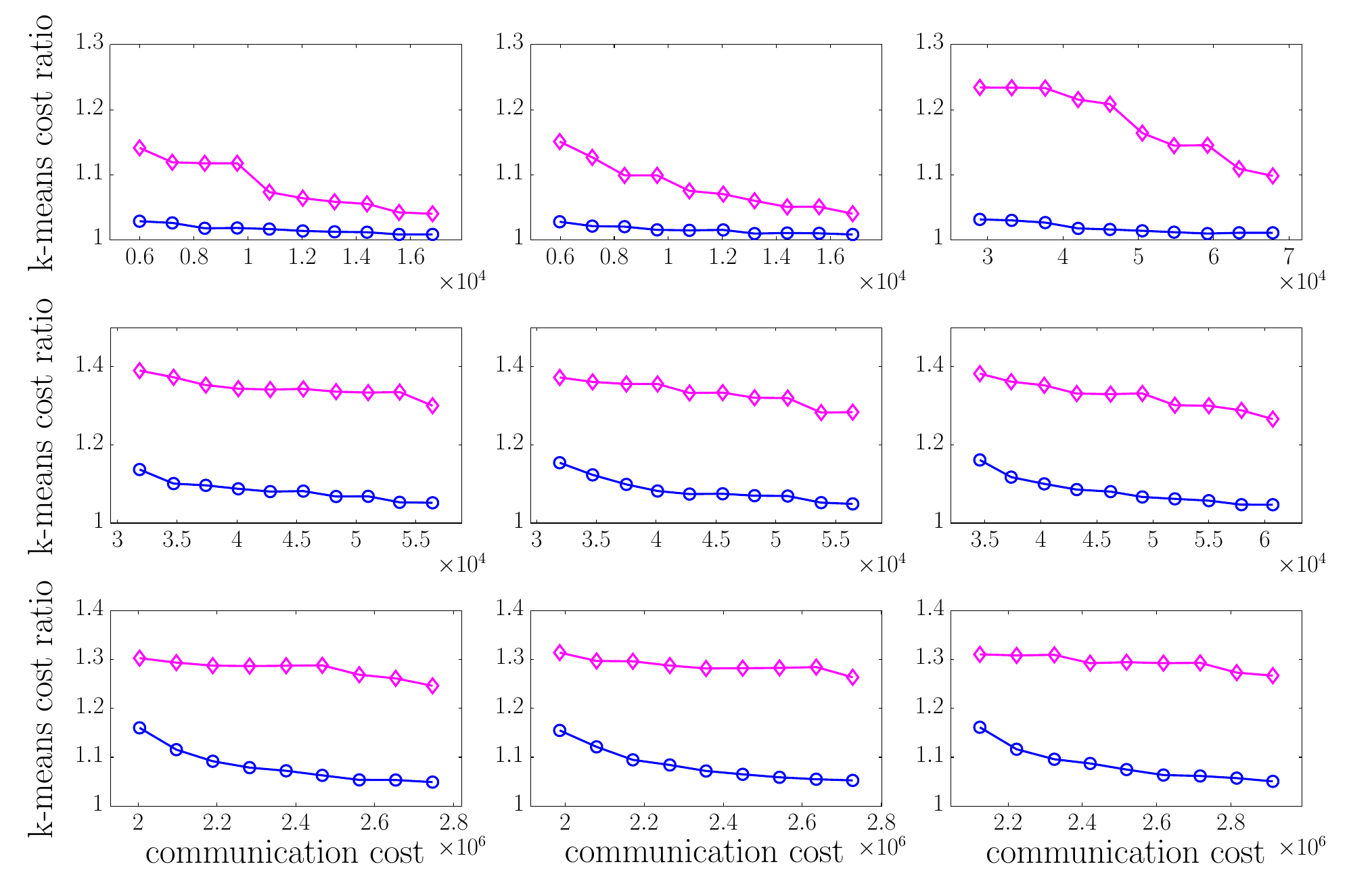}
\caption{$k$-means cost on the spanning trees of the grid and preferential graphs. Columns: grid graph with similarity-based partition, grid graph with weighted partition, and preferential graph with degree-based partition. Rows: Spam, Pendigits, Letter, synthetic, ColorHistogram, and YearPredictionMSD.}\label{fig:cost_tree_grid}
\end{minipage}
\end{figure*}

\end{document}